\documentclass{new_tlp}
\usepackage{mathptmx}
\usepackage{amsfonts}
\def\definition{def-module\xspace}

\newcommand\bcmdtab{\noindent\bgroup\tabcolsep=0pt%
	\begin{tabular}{@{}p{10pc}@{}p{20pc}@{}}}
	\newcommand\ecmdtab{\end{tabular}\egroup}

\def\univ{\widetilde{\forall}}

\title[]
{Strong Equivalence and Program Structure in Arguing Essential Equivalence between\\ Logic Programs} 

\author[Yu. Lierler]
{YULIYA LIERLER\\
	6001 Dodge St, Omaha, NE 68182, USA\\
	\email{ylierler@unomaha.edu}}

\jdate{March 2003}
\pubyear{2003}
\pagerange{\pageref{firstpage}--\pageref{lastpage}}
\doi{S1471068401001193}

\usepackage[english]{babel}
\usepackage[autostyle]{csquotes}
\usepackage{times}
\usepackage{xcolor}
\usepackage{soul}
\usepackage[utf8]{inputenc}
\usepackage[small]{caption}
\usepackage{enumitem}
\usepackage{url}  
\usepackage{graphicx}  
\usepackage{tikz}
\def\C{\cal C}
\renewcommand{\vec}[1]{\mathbf{#1}}

\usepackage{mathptmx}
\setlist[itemize]{noitemsep, topsep=0pt, leftmargin=*}
\usepackage{amsmath}

\def\SM{\hbox{SM}}

\def\pred{\pi}

\def\G{\mathcal{G}}

\def\body{\mathit{Body}}

\DeclareMathOperator*{\bigcomma}{\textrm{\huge ,}}
\DeclareMathOperator*{\bigmid}{\textrm{\huge {$\mid$}}}

\usepackage{xspace}
\usepackage{color}
\usepackage{tikz}\usetikzlibrary{arrows}
\tikzset{
	>=stealth,
	auto,
	node distance=3.5cm,
	font=\scriptsize,
	possible world/.style={circle,draw,thick,align=center},
	real world/.style={double,circle,draw,thick,align=center},
	minimum size=40pt
}

\usepackage{todonotes}
\usetikzlibrary{positioning, automata}
\usepackage{verbatim}
\tikzset{place/.style={circle, draw}}
\tikzset{>=stealth, auto, node distance=2.5cm, every loop/.style={->, min distance=10mm, in=0, out=60, looseness=10}}

\def\bp{\textbf{p}}

\def\observation{{\em Claim}\xspace}
\def\observations{\hbox{\observation}\hbox{{\em s}}}

\def\caused{\hbox{$\mathbf{caused}\ $}\xspace}
\def\causes{\hbox{$\mathbf{causes}\ $}\xspace}
\def\inertial{\hbox{$\mathbf{inertial}\ $}\xspace}

\def\ifc{\hbox{$\ \mathbf{if}\ $}\xspace}
\def\after{\hbox{$~\mathbf{after}~$}\xspace}

\def\ar{\leftarrow}
\def\rar{\rightarrow}

\def\lrar{\leftrightarrow}
\def\beq{\begin{equation}}
\def\eeq#1{\label{#1}\end{equation}}
\def\ba{\begin{array}}
\def\ea{\end{array}}

\def\t#1{\hbox{$\hbox{-}{#1}$}}
\def\h#1{\hbox{$\widehat{#1}$}}
\def\hh#1{\hbox{$\widetilde{#1}$}}

\def\fol#1{\hbox{$\widehat{#1}$}}

\def\o#1{\hbox{$\overline{#1}$}}

\def\ii#1{\hbox{\it #1\/}}

\def\definition{def-module\xspace}

\def\naf{\hbox{\em not }}

\def\after{\hbox{\bf after}}

\def\causes{\hbox{\bf causes}}
\def\inertial{\hbox{\bf inertial}}

\def\alchoice{{\em Plan-choice}\xspace}
\def\planinst{{\em Plan-instance}\xspace}

\def\dalchoice{{\em Plan-disj}\xspace}

\def\sqht{\hbox{\bf SQHT$^=$}\xspace}

\def\C{\hbox{$\mathcal{C}$}\xspace}
\def\al{\hbox{$\mathcal{AL}$}\xspace}

\def\D{\Delta}
\def\seq{\Rightarrow}
\def\r#1#2{\frac{\textstyle #1}{\textstyle #2}}

\def\citep#1{\cite{#1}}

\def\r#1#2{\frac{\textstyle #1}{\textstyle #2}}

\def\sig#1{atoms(#1)}

\def\pred{\pi}

\def\shift{\hbox{$\mathit{shift}$}}
\def\location{loc\xspace}
\def\fluent{fluent\xspace}
\def\step{step\xspace}
\def\inc{inc\xspace}

\def\occurs{o\xspace}
\def\block{block\xspace}
\def\noccurs{non\_o\xspace}
\def\holds{h\xspace}
\def\nholds{non\_\holds\xspace}

\def\action{action\xspace}
\def\sthHpd{\hbox{$\mathit{sthHpd}$}\xspace}

\newtheorem{prop}{Proposition} 
\newtheorem{lemma}{Lemma} 
\newtheorem{theorem}{Theorem} 
\newtheorem{corollary}{Corollary} 

\usepackage{times}  
\usepackage{helvet}  
\usepackage{courier}  
\usepackage{url}  
\usepackage{graphicx}  
\makeatletter
\def\thm@space@setup{%
  \thm@preskip=\parskip \thm@postskip=0pt
}
\makeatother

\begin{document}
\date{}	

\maketitle

\begin{abstract}
Answer set programming  is a prominent declarative programming paradigm used in formulating combinatorial search problems and implementing different knowledge representation formalisms. 
Frequently, several related and yet substantially different answer set programs exist for a given problem. Sometimes these encodings may display significantly different performance. Uncovering {\em precise formal} links between these programs is often important and yet far from trivial. 
This paper presents formal results carefully relating  a number of interesting program rewritings.
 It also provides the proof of correctness of system {\sc projector} concerned with automatic program rewritings for the sake of efficiency.
 Under consideration in Theory and Practice of Logic Programming (TPLP)
\end{abstract}
\section{Introduction}	

 Answer set programming (ASP) is a prominent knowledge representation paradigm 
 with roots in logic programming \cite{bre11}. It is frequently used for addressing 
 combinatorial search problems. 
  It has also been used to provide implementations and/or translational semantics to other knowledge representation formalisms such as action languages including languages~$\mathcal{B}$~\cite[Section 5]{gel98},~\C~\citep{lif99b}, $\mathcal{BC}$~\citep{lee2013action},~$\C+$~\citep{giu04,bab13}, and~\al~\cite[Section 8]{gel14}.

In answer set programming, a given computational problem 
is represented by a {\em declarative program}, also called a {\em problem encoding}, that describes the properties of  
a solution to the problem. Then, an answer set solver is used to generate 
answer sets for the program. These answer sets correspond to solutions to 
the original problem.  
As answer set programming evolves, new language features come to life providing means to reformulations of original problem encodings. Such new formulations often prove to be more intuitive and/or more concise and/or more efficient. Similarly, when a software engineer tackles a problem domain by means of answer set programming it is a common practice to first develop {\em a/some} solution to a problem and then rewrite this solution iteratively using such techniques, for example, as projection to  gain a better performing encoding~\citep{bud15}. These common processes bring a scientific question to light: {\em what are the formal means to argue the correctness of renewed formulations of the original encodings to problems?} In other words,  under the assumption that the original encoding to a problem is correct, how can we argue that a related and yet  different encoding  is also correct?
In addition, automated ASP program rewriting systems come to life.
Systems~{\sc lpopt}~\cite{bic16} and ~{\sc projector}~\cite{hip18}   exemplify such a trend. These tools rewrite an original program into a new one with the goal of improving an ASP solver's performance. Once again, formal means are necessary to claim the correctness of such systems. We note that the last section of this work is devoted to the claim of correctness of system~{\sc projector}.

It has been long recognized that studying various notions of equivalence
between programs under the answer set semantics is of crucial importance. 
Researchers proposed and studied  strong equivalence~\citep{lif01,lif07a}, uniform equivalence~\citep{eit03},
relativized  strong and uniform equivalences~\citep{wol04}. Another related approach is the study of forgetting~\cite{lei17}.
Also,    equivalences relative to specified signatures were considered~\citep{erdo04,eit05,wol08,har16}. In most  of the cases the programs considered for studying the distinct forms of equivalence are propositional.
Works~by~\citeN{eit06}, \citeN{eit06a},
\citeN{lif07a},
\citeN{oet08},
\citeN{pea12}, and
\citeN{har16} are exceptions. These authors consider programs with variables (or, first-order programs). It is first-order programs that  ASP knowledge engineers develop. Thus, theories on equivalence between programs with variables  are especially important as they can lead to more direct arguments about properties of programs used in practice.

In this paper we show how concepts of strong equivalence and so called conservative extension are of use in illustrating that two programs over different signatures and with significantly different structure are ``essentially the same'' or ``essentially equivalent'' in a sense that they capture solutions to the same problem. 
Let us make the concept of an essential equivalence between problem's encodings precise. 
We use the same notion of the search problem as Brewka et al.~\shortcite{bre11}. Quoting from their work, a {\em search problem}~$P$ consists of a set of instances with each
{\em instance}~$I$ assigned a finite set~$S_P(I)$ of solutions. 
We say that logic program $\Pi_P(\cdot)$ is an encoding of $P$ when for any instance $I$ of this problem, the solutions to $I$ -- the elements of set~$S_P(I)$ -- can be reconstructed from the answer sets of  program $\Pi_P(I)$. We say that encodings $\Pi_P(\cdot)$ and $\Pi'_P(\cdot)$ are {\em essentially equivalent}, 
 when given any instance $I$ of  problem $P$ the answer sets of programs 
 $\Pi_P(I)$ and $\Pi'_P(I)$    are in one-to-one correspondence.
The paper has two parts. In the \textit{first part} we consider  propositional programs. In the \textit{second part}, we move  to the programs with variables. \textit{These parts can be studied separately}. \textit{The first one is appropriate for} researchers who are not yet deeply familiar with answer set programming theory and are interested in learning formal details. \textit{The second part\footnote{This part of the 
paper is a substantially extended version of the paper presented at PADL 2019~\cite{lier18}.} is geared towards} answer set programming practitioners providing them with  theoretical grounds and tools to assist them in program analysis and formal claims about the developed encodings and their relations. In both of these parts we utilize running examples stemming from the literature. For instance, for the case of propositional programs we study two distinct ASP formalizations of action language~$\C$.  In the case of first-order programs, we study two distinct formalizations of planning modules for action language \al given in~\cite[Section 9]{gel14}. Namely,
\begin{enumerate}[topsep=0pt]
	\itemsep0em 
	\item a \alchoice formalization that utilizes choice rules and aggregate expressions,
	\item a \dalchoice formalization that utilizes disjunctive rules.
\end{enumerate} 
  In both cases we identify interesting results. 

\paragraph{\bf Paper Outline.}
The paper is structured  as follows. Section~\ref{sec:lp} presents the concepts of (i) a propositional logic program, (ii) strong equivalence between propositional logic programs, and (iii) a propositional logic program being a conservative extension  of another one. 
Section~\ref{sec:edef} introduces a rewriting technique frequently used by ASP developers when a new auxiliary proposition is introduced in order to denote a conjunction of other propositions. 
Then these conjunctions are safely renamed by the auxiliary atom. We refer to this process as explicit definition rewriting and illustrate its correctness.
We continue by reviewing  action language \C in Section~\ref{sec:actionlangC}, which serves a role of a running example in  the first part of the paper.
In Section~\ref{sec:btr}, we present an original, or gold standard, translation of language \C to a logic program. 
Section~\ref{sec:sbtr} states a modern formalization stemming from the translation of a syntactically restricted fragment of $\C+$. 
At last, in Section~\ref{sec:relation}
we showcase how we can  argue on the correctness of a modern formalization  by stating the formal relation between the original and modern translations of language \C. An interested reader may proceed to the Appendix to find the details of the proof of the claim. There, we  utilize reasoning by ''weak'' natural deduction and a formal result on explicit definition rewriting. (A weak natural deduction system is reviewed in the Appendix.)
We also note that there is an interesting by product of our analysis: we establish that $\C+$ can be viewed as a true generalization of the language~$\C$ to the case of multi-valued fluents.

We start the second part of the paper  by presenting the  \alchoice and \dalchoice programs at the onset of Section~\ref{sec:foprogram}.
We then introduce the logic program language called RASPL-1 in Section~\ref{sec:sm}.  The semantics of this language is given in terms of the SM operator reviewed in Section~\ref{sec:sem}. In Section~\ref{sec:strong}, we review the concept of strong equivalence for first order programs. Section~\ref{sec:rewr} is devoted to a sequence of formal results on  program rewritings. 
One of the findings of this work is lifting the results by 
\citeN{ben94} to the first order case.
Earlier work claimed that propositional head-cycle-free disjunctive programs can be rewritten to nondisjunctive programs by means of simple syntactic transformation. Here we not only generalize this result to the case of first-order programs, but also illustrate that at times we can remove disjunction from parts of a program even though the program is not head-cycle-free. Another important contribution of the work is lifting the 
Completion Lemma and the Lemma on Explicit Definitions
stated in~\citep{fer05,fer05b} from the case of propositional theories and propositional logic programs to first-order programs. In conclusion, in Section~\ref{sec:proj} we review a frequently used rewriting technique called projection that often produces better performing encodings. We illustrate the utility of the presented theoretical results as they can be used to argue the correctness of distinct versions of projection that also include rules with aggregates. In particular, the last formal result stated in the paper provides a proof of correctness of system~{\sc projector}. The Lemma on Explicit Definitions presented here is essential in this argument.

The Appendix provides the proofs for the formal results presented in the paper. 

\section{Propositional Programs}\label{sec:prop}

\subsection{Traditional Logic Programs and their Equivalences}\label{sec:lp}
A {\em (traditional logic) program} is a finite set of {\em rules}  of the form
\beq
\ba{l}
a_0 \ar a_{1} , \dots , a_l , not\  a_{l+1} , \dots , not\  a_m ,  not\  not\  a_{m+1} , \dots , 
not\  not\  a_n, 
\ea
\eeq{eq:ruletraditional}
$(0\leq l \leq m \leq n)$, where each $a_0$ is an atom or $\bot$ and each $a_i$ ($1\leq i\leq n$) is an atom, $\top$, or $\bot$ .
The expression containing atoms 
$a_{1}$ through $a_{n}$ is called the {\em body} of the rule. 
Atom $a_0$ is called a {\em head}.

We define the {\em answer sets} of
a traditional program $\Pi$ following~\citep{lif99d}. We say that a program is {\em basic}, when it does not contain connective $not$. 
In other words a basic program consists of  rules
\beq
a_0 \ar a_{1} , \dots , a_l, 
\eeq{eq:basic}
where each $a_0$ is an atom or $\bot$ and each $a_i$ ($1\leq i\leq l$) is an atom, $\top$, or $\bot$. We say that a set $X$ of atoms {\em satisfies}
rule~\eqref{eq:basic} if it satisfies the implication
$$a_{1} \wedge \cdots \wedge a_l\rightarrow a_0.$$  
We say that a set $X$ of atoms is an {\em answer set} of a basic program $\Pi$ if $X$ is a minimal set among sets satisfying all rules of $\Pi$.

A {\em reduct} of a program $\Pi$ with respect to a set $X$ of atoms, denoted by $\Pi^X$, is constructed as follows.
For each rule~\eqref{eq:ruletraditional} in $\Pi$
\begin{enumerate}
\item when   $not\ not\ a_i$ ($m+1\leq i\leq n$) is such that $a_i\in X$, replace this expression with~$\top$, otherwise replace it with $\bot$,
\item when   $not\ a_i$ ($l+1\leq i\leq m$) is such that $a_i\in X$, replace this expression with~$\bot$, otherwise replace it with $\top$.
\end{enumerate}
It is easy to see that a reduct of a program forms a basic program. We say that a set $X$ of atoms is an {\em answer set} of a traditional program if it is an answer set for the reduct~$\Pi^X$.

In the later part of the paper we present the definition of an answer set for programs with variables by means of operator SM~\cite{fer09}. \citeN{fer09} show  in which sense SM operator captures the semantics of answer sets presented here. 

According 
to~\citep{fer05b} and \citep{fer05}, rules of the form~\eqref{eq:ruletraditional} are sufficient 
to capture the meaning of the choice rule construct commonly used in 
answer set programming. For instance, the choice rule ${\{p\}\ar q}$ 
is understood as the rule 
$$p\ar q,\ not\ not\ p.$$
We intuitively read this rule as {\em given $q$ atom $p$ {\em may be} true.} 
We  use  choice rule notation in the sequel. 
\paragraph{Strong Equivalence}
Traditional programs $\Pi_1$ and~$\Pi_2$ are {\em strongly equivalent}~\citep{lif01} when for 
every program $\Pi$, programs $\Pi_1\cup \Pi$ and $\Pi_2\cup \Pi$ have the same answer sets. In addition to introducing  strong equivalence, \citeN{lif01} also illustrated that traditional programs can be associated with the propositional formulas and a question whether the programs are strongly equivalent can  be turned into a question whether the respective propositional formulas are equivalent in the logic of here-and-there (HT-logic)~\cite{Luk38}, an intermediate logic between classical and intuitionistic logics.

We follow the steps of \citep{lif01} and identify a rule~\eqref{eq:ruletraditional} with the propositional formula
\beq
\ba{r}
a_{1} \wedge \dots \wedge a_l \wedge \neg  a_{l+1} \wedge \dots 
\wedge \neg a_m \wedge   \neg \neg a_{m+1}
\wedge \dots \wedge \neg\neg  a_n \rar a_0.  
\ea
\eeq{eq:formulatraditional}

\paragraph{Conservative Extensions}
\citeN{har16} defined the notion of a conservative extension
for the case of logic programs.
Similarly to  strong equivalence, it attempts to capture the conditions under which we  can rewrite parts of the program and yet guarantee that the resulting program is  not different in an essential way from the original one.
Conservative extensions allow  
us to reason about rewritings even when the rules in question have 
different signatures.

For a program $\Pi$, by $\sig{\Pi}$ we denote the set of atoms occurring in $\Pi$.
Let~$\Pi$ and~$\Pi'$ be programs such that
$\sig{\Pi}\subseteq\sig{\Pi'}$.
We say that program~$\Pi'$ is  
a {\em conservative extension} of $\Pi$ if $X \mapsto X\cap\sig{\Pi}$ is a 1-1 
correspondence between the answer sets of $\Pi'$ and the answer sets of $\Pi$. 
For instance, program
\beq
\ba{l}
\neg q\rar p\\
\neg p\rar q
\ea
\eeq{eq:prpq}
is a conservative extension of the program containing the single choice rule
$$
\{p\}.
$$
Furthermore, given  program $\Pi$ such that (i) it contains rule $\{p\}$ and (ii) $q\not\in\sig{\Pi}$, a program constructed from $\Pi$ by replacing~$\{p\}$ with~\eqref{eq:prpq} is a conservative extension of~$\Pi$.

\subsection{On Explicit Definition Rewriting}\label{sec:edef} 
We now turn our attention to a common rewriting technique based on  explicit definitions and illustrate its correctness. This technique introduces an  auxiliary proposition  in order to denote a conjunction of other propositions. Then these conjunctions are safely renamed by the auxiliary atom in the remainder of the program.

We call a formula {\em basic conjunction} when it is of the form
\beq
a_{1} \wedge \dots \wedge a_l \wedge \neg  a_{l+1} \wedge \dots 
\wedge \neg a_m \wedge    \neg \neg a_{m+1}
\wedge \dots \wedge \neg\neg  a_n,\eeq{eq:bconjunction}
where each $a_i$ ($1\leq i\leq n$) is an atom, $\top$, or $\bot$. For example, the body of any rule in a traditional program is a basic conjunction.

Let $\Pi$ be a program,  $Q$ be a set of atoms that do not occur in $\Pi$. For an atom $q\in Q$, let $def(q)$ denote a basic conjunction~\eqref{eq:bconjunction}, 
where $a_i$ ($1\leq i\leq n$) in $\sig{\Pi}$. 
We say that $def(q)$ is an {\em explicit definition} of $q$ in terms of $\Pi$. 
By $def(Q)$ we denote a set of formulas $def(q)$ for each atom $q\in Q$. We assume that all these formulas are distinct.
Program $\Pi[Q,def(Q)]$ is constructed from~$\Pi$ as follows:
\begin{itemize}
	\item all occurrences of all formulas  $def(q)$ from $def(Q)$ in some body of $\Pi$ are replaced by respective $q$,
	\item for every atom  $q\in Q$ a rule of the form 
	$$def(q)\rar q$$
	is added to the program. 
\end{itemize}

For instance,let $\Pi$ be a program
$$
\neg q \rar p
$$
and $def(r)$ be a formula $\neg q$,
then $\Pi[\{r\},\{def(r)\}]$  follows
$$
\ba{l}
r \rar p\\
\neg q \rar r
\ea
$$
The proposition below supports the fact that the latter program  is a conservative extension of the former. It is an important claim as although this kind of rewriting is very frequently used in practice to the best of our knowledge this is the first time it has been formally claimed.

\begin{prop}\label{prop:ce}
	Let $\Pi$ be a program,  $Q$ be a set of atoms that do not occur in $\Pi$, and~$def(Q)$ be a set composed of explicit definitions for each element in $Q$ in terms of~$\Pi$.
	Program $\Pi[Q,def(Q)]$ is a conservative extension of $\Pi$.
\end{prop}

\subsection{Review of Action Language C}\label{sec:actionlangC}
This review of action language \C follows~\citep{lif99b}.

We consider a set $\sigma$ of propositional symbols partitioned into the fluent names $\sigma^{fl}$
and the elementary action names~$\sigma^{act}$. An action is an interpretation of $\sigma^{act}$.
Here we only consider what \citeN{lif99b} call {\em definite} action descriptions so that we only define this special class of~\C action descriptions.

Syntactically, a \C~ {\em action description} is a set of {\em static} and {\em dynamic laws}.
{\em Static laws} are of the form
\beq
\caused l_0 \ifc l_1\wedge\cdots\wedge l_m 
\eeq{eq:slaw} 
and dynamic laws are of the form
\beq
\caused l_0 \ifc l_1\wedge\cdots\wedge l_m \ \after\ l_{m+1}\wedge\cdots\wedge l_n 
\eeq{eq:dlaw}
where 
\begin{itemize}
	\item  $l_0$ is either a literal over $\sigma^{fl}$ or the symbol $\bot$, 
	\item  $l_i$ ($1\leq i\leq m$) is a literal in $\sigma^{fl}$, 
	\item  $l_i$ ($m+1\leq i\leq n$) is a literal in $\sigma$ , and
	\item conjunctions $l_1\wedge\cdots\wedge l_m$ and $l_{m+1}\wedge\cdots\wedge l_n$ are possibly empty and understood as~$\top$ in this case.
\end{itemize}
In both laws, the literal $l_0$ is called the {\em head}. 

Semantically, an action description defines a graph or a transition system. 
We call nodes of this graph {\em states} and directed edges {\em transitions}. 
We now define these concepts precisely. 
Consider an action description $D$. A {\em state} is an interpretation of~$\sigma^{fl}$ that
satisfies implication 
$$l_1\wedge\cdots\wedge l_m\rightarrow l_0$$ 
for every static law~\eqref{eq:slaw}  in $D$. 
A {\em transition} is any triple $\langle s, a, s' \rangle $, where $s$, $s'$ are states
and $a$ is an action; $s$ is the {\em initial} state of the transition,
and $s'$ is its {\em resulting} state. A literal $l$ is {\em caused} in a transition 
$\langle s, a, s' \rangle $
if it is
\begin{itemize} 
	\item  the head of a static law~\eqref{eq:slaw} from $D$ such that $s'$ satisfies $l_1\wedge\cdots\wedge l_m$, or
	\item the head of a dynamic law~\eqref{eq:dlaw} from $D$ such that $s'$ satisfies $$l_1\wedge\cdots\wedge l_m$$  and $s\cup a$ satisfies $$l_{m+1}\wedge\cdots\wedge l_n.$$
\end{itemize}
A transition  $\langle s, a, s' \rangle $ is {\em causally explained} by $D$ if its resulting state $s'$ is the set of literals caused  in this transition.

The {\em transition system described by an action description}~$D$ is the directed
graph, which has the states of $D$ as nodes, and which includes an edge from state $s$ to
state $s'$ labeled $a$ for every transition  $\langle s, a, s' \rangle $ that is causally explained by $D$.

We now present an example by~\citeN{lif99b} that formalizes {\em the effects of putting an object in water}. 
We use this domain as a running example.
It uses the fluent names $inWater$ and $wet$ and the elementary action name $putInnWater$. 
We follow the convention by~\citeN{lif99b} and present states (interpretations) as lists of literals.
In the notation  introduced by \citeN[Section 6]{gel98}, the action description for {\em water domain} follows\footnote{We remark on the keyword ${\inertial}$.
	It intuitively suggests that a fluent declared to be inertial is such that its value  can be changed by  actions only. If no actions, which directly or indirectly affect such a fluent, occur then the value of the inertial fluent remains unchanged.
} 
$$
\ba{l}
\caused wet \ifc inWater\\
putInWater\ \causes\ inWater\\
\inertial\    inWater, \neg inWater, wet, \neg wet
\ea
$$
Written in full this action description contains six laws:
$$
\ba{l}
\caused wet \ifc inWater\\
\caused inWater \ifc \top\ \after\ putInWater \\
\caused inWater \ifc inWater\ \after\ inWater\\
\caused \neg inWater \ifc \neg inWater\ \after\ \neg inWater\\
\caused wet \ifc wet\ \after\ wet\\
\caused \neg wet \ifc \neg wet\ \after\ \neg wet\\
\ea
$$
The corresponding transition system has 3 states:
$$\neg inWater~\neg wet,~\neg inWater~wet, ~inWater~wet
$$
and 6 causally explained transitions
\beq
\small
\ba{ll}
\langle \neg inWater~\neg wet , \neg putInWater, \neg inWater~\neg wet\rangle,&
\langle \neg inWater~\neg wet ,  putInWater,  inWater~ wet\rangle,\\
\langle \neg inWater~wet , \neg putInWater, \neg inWater~ wet\rangle,&
\langle \neg inWater~wet ,  putInWater,  inWater~ wet\rangle,\\
\langle  inWater~wet , \neg putInWater,  inWater~ wet\rangle,&
\langle  inWater~wet ,  putInWater,  inWater~ wet\rangle.\\
\ea
\normalsize
\eeq{eq:cexample}
We depict this transition system in Figure~\ref{fig:trsys}.
\begin{figure}[th]
\small
	\begin{tikzpicture}
	\tikzset{node distance=3.5cm, auto}
	\node  [possible world] (0) {$~\neg inWater$\\ $wet$};
	\node  [possible world] (1) [right of=0, xshift = .3cm] {$\neg inWater$\\$\neg wet$};
	\node  [possible world] (2) [below of=0, xshift = 2.3cm,  yshift = 1cm,] {$inWater$\\$ wet$};
	
	\path[]
	(0) edge [loop right] node {$\neg putInWater$} (0)
	(0) [->] edge node[left] {$putInWater$} (2)
	(1) edge [loop right] node {$\neg putInWater$} (1)
	(2) edge [loop right] node {$\neg putInWater,~~$  $putInWater$} (2)
	(1) [->] edge node[left] {$putInWater$} (2);
	\end{tikzpicture}
	\normalsize
	
	\caption{Transition diagram for Water domain.\label{fig:trsys}} \end{figure}

\subsection{On Relation between the Original and Modern Formalizations of \C}\label{sec:umbrella}
We start this section by reviewing the original formalization of action language~\C in the language of logic programs under answer set semantics~\citep{lif99b}. 
Specifically, \citeN{lif99b} proposed a translation from an action description $D$ in~\C to a logic program $lp_T(D)$ so that the answer sets of this program capture all the "histories" of length $T$ in the transition system specified by $D$. 

Since that original work, languages of answer set programming have incorporated  new features such as, for instance, choice rules. At present, these are commonly used by the practitioners of ASP. It is  easy to imagine that in a modern formalization of action language~\C, given a system description~$D$ a resulting program will be different from the original  $lp_T(D)$. In fact,~\citeN{bab13} present a translation of an action language {\C}$+$ (according to \citeN[Section~7.3]{giu04}  $\C$ is {\em the immediate predecessor} of $\C+$) that utilizes modern language features such as choice rules. In Section~\ref{sec:sbtr}, we present this translation  for the case of \C. In particular, we 
restrict the language of $\C+$ to Boolean, or two-valued, fluents (in general, $\C+$ permits multivalued fluents).
We call this translation~$simp_T(D)$. 
Although,  $lp_T(D)$ and $simp_T(D)$ share a lot in common they are substantially different. To begin with, the signatures of these programs are not identical. Also, $simp_T(D)$ utilizes choice rules. The programs $lp_T(D)$ and $simp_T(D)$ are different enough that it is not immediately obvious that their answer sets capture the same entities. There are two ways to argue that the program 
$simp_T(D)$ is ``essentially the same'' as program $lp_T(D)$:
to illustrate that the answer sets of $simp_T(D)$ capture all the "histories" of length~$T$ in the transition system specified by $D$ by relying
\begin{enumerate} 
	\item   on the definitions of action language~\C;
	\item   on the properties of 
	programs $lp_T(D)$ and $simp_T(D)$ that establish a one-to-one correspondence between their answer sets.
\end{enumerate}

Here we take the second way into consideration. We illustrate how the concepts of strong equivalence and conservative extension together with formal results previously discovered about these 
prove to be of essence in this argument. The details of this argument are given in the Appendix. Thus, we showcase a proof technique for arguing on the correctness of a logic program. This proof technique assumes the existence of a ''gold standard'' logic program formalizing a problem at hand, in a sense that this gold standard is trusted to produce correct results. 
It is a common practice in  development of answer set programming solutions  to obtain a final formalization of a problem by first producing such a gold standard program and then applying a number of rewriting procedures to that program to enhance its performance. The benefits of the proposed method are twofold. 
First, this methodology can be used by software engineers during a formal analysis of their solutions. Second, we trust that this methodology paves a way for a general framework for arguing correctness of common program rewritings so that they can be automated for the gain of performance. This is a question for investigation in the future.

\subsubsection{Review of Basic Translation}\label{sec:btr}

Let $D$ be an action description. \citeN{lif99b} defined a translation from action description $D$ to a logic program $lp_T(D)$ parametrized with a positive integer~$T$ that intuitively represents a time horizon. The remarkable property of logic program  $lp_T(D)$ that its answer sets correspond to "histories" -- path/trajectories of length $T$ in the transition system described by $D$.

Recall that by $\sigma^{fl}$ we denote fluent names of $D$ and  by $\sigma^{act}$ we denote elementary action names of $D$. Let us construct "complementary" vocabularies to $\sigma^{fl}$ and $\sigma^{act}$ as follows
$$\t{\sigma^{fl}}=\{\t{a}\mid a\in \sigma^{fl}\}$$ 
and 
$$\t{\sigma^{act}}=\{\t{a}\mid a\in \sigma^{act}\}.$$

For a literal $l$, we define 
$$\h{l}=\begin{cases}
a\hbox{ if $l$ is an atom $a$}\\
\t{a}\hbox{ if $l$ is a literal of the form $\neg a$}\\
\end{cases}
$$
and
$$\o{l}=\begin{cases}
\t{a}\hbox{ if $l$ is an atom $a$}\\
{a}\hbox{ if $l$ is a literal of the form $\neg a$}\\
\end{cases}
$$

The language of $lp_T(D)$ has atoms of four kinds:
\begin{enumerate}
	\item fluent atoms--the fluent names of $\sigma^{fl}$ followed by $(t)$ where $t = 0,\dots, T$, 
	\item action atoms--the action names of $\sigma^{act}$ followed by $(t)$ where $t = 0,\dots, T-1$, 
	\item complement fluent atoms--the elements of $\t{\sigma^{fl}}$ followed by $(t)$ where $t = 0,\dots, T$,
	\item complement action atoms--the elements of $\t{\sigma^{act}}$ followed by $(t)$ where \hbox{$t = 0,\dots, T-1$}.
\end{enumerate}
Program $lp_T(D)$ consists of the following rules:
\begin{enumerate}
	\item for every atom $a$ that is a fluent or action atom of the language of  $lp_T(D)$
	\beq
	\bot \ar a,\t{a}
	\eeq{eq:consistency}
	and 
	\beq
	\bot\ar\ not\  a,\ not\  \t{a}
	\eeq{eq:complete}
	\item for every static law~\eqref{eq:slaw} in $D$, the rules 
	\beq
	\h{l_0}(t) \ar\ not\ \o{l_1}(t),\dots,\ not\ \o{l_m}(t) 
	\eeq{eq:rulestatic}
	for all $t=0,\dots,T$ (we understand $\h{l_0}(t)$ as $\bot$ if $l_0$ is $\bot$),
	\item for every dynamic law~\eqref{eq:dlaw} in $D$, the rules 
	\beq
	\ba{r}
	\h{l_0}(t+1) \ar\ not\ \o{l_1}(t+1),\dots,\ not\ \o{l_m}(t+1),
	\h{l_{m+1}}(t),\dots,\ \h{l_n}(t), 
	\ea
	\eeq{eq:dynamicorig}
	for all $t=0,\dots,T-1$,
	\item the rules
	\beq
	\ba{l}
	\t{a}(0)\ar\ not\ a(0)\\ 
	a(0)\ar\ not\ \t{a}(0),\\ 
	\ea
	\eeq{eq:fl0}
	for all fluent names  $a$ in $\sigma^{fl}$ and
	\item  for every atom $a$ that is an action atom of the language of  $lp_T(D)$
	the rules
	$$
	\ba{l}
	\t{a}\ar\ not\ a\\ 
	a\ar\ not\ \t{a}.\\ 
	\ea
	$$ 
\end{enumerate}
\begin{prop}[Proposition~1 in~\citep{lif99b}]~\label{prop:lif} For a set $X$ of atoms,~$X$ is an answer set for $lp_T(D)$ if and only if it has the form 
	
	$$
	\Bigg[\bigcup_{t=0}^{T-1}{\{\h{l}(t)\mid l\in s_t\cup a_t\}}\Bigg]~\cup~\{\h{l}(T)\mid l\in s_T\}
	$$
	for some path $\langle s_0,a_0,s_1,\dots,s_{T-1},a_{T-1},s_T\rangle$ in the transition system described by~$D$.
\end{prop}
We note that \citeN{lif99b} presented $lp_T$ translation and Proposition~1 using both default negation $not$ and classical negation $\neg$ in the program. Yet, classical negation 
can always be eliminated from a program by means of auxiliary atoms and additional constraints as it is done here. In particular, auxiliary atoms have the form $\t{a}(i)$ (where $\t{a}(i)$ intuitively stands for literal~$\neg a(i)$), while the additional constraints have the form~\eqref{eq:consistency}.

To exemplify this translation, consider $\C$ action description~\eqref{eq:cexample}. Its
 translation  consists of all rules of the form
$$
\ba{ll|ll}
1. &\bot\ar inWater(t),\t{inWater}(t) &2.
&wet(t)\ar\ not\ \t{inWater}(t)\\
&\bot\ar\ not\ inWater(t),\ not\ \t{inWater}(t)&&\\

&\bot\ar wet(t),\t{wet}(t)\\
&\bot\ar\ not\  wet(t),\ not\ \t{wet}(t)&&\\

&\bot\ar putInWater(t),\ not\ \t{putInWater}(t)&&\\
&\bot\ar\ not\  putInWater(t),\ not\ \t{putInWater}(t)&&\\
\ea$$
$$
\ba{ll|ll}
3.& inWater(t+1)\ar\ {putInWater}(t)&4.
&\t{inWater}(0)\ar\ not\ {inWater(0)}\\
&inWater(t+1)\ar\ not\ \t{inWater}(t+1),\ {inWater}(t)&&inWater(0)\ar\ not\ \t{inWater(0)}\\

&\t{inWater}(t+1)\ar\ not\ {inWater}(t+1), \t{inWater}(t)&&\t{wet}(0)\ar\ not\ {wet(0)}\\
&wet(t+1)\ar\ not\ \t{wet}(t+1), \ {wet}(t)&&wet(0)\ar\ not\ \t{wet(0)}\\
&\t{wet}(t+1)\ar\ not\ {wet}(t+1), \t{wet}(t)\\
\hline

5.
&\t{putInWater}(t)\ar\ not\ {putInWater(t)}&&\\
&putInWater(t)\ar\ not\ \t{putInWater(t)}&&\\
\ea
$$

\subsubsection{Simplified Modern Translation}\label{sec:sbtr}
As in the previous section, let $D$ be an action description and $T$ a positive integer. 
In this section we define a translation from action description $D$ to a logic program $simp_T(D)$ inspired by $lp_T(D)$ and the advances in answer set programming  languages. The main property of logic program  $simp_T(D)$ is as in case of 
$lp_T(D)$ 
that its answer sets correspond to histories captured by the transition system described by $D$.  This translation is a special case of a translation by~\citeN{bab13} for an action language {\C}$+$ 
that is limited  to two-valued fluents.

The language of $simp_T (D)$ has atoms of three kinds that coincide with the three first groups~(1-3) of atoms identified in the language of  $lp_T(D)$.

For a literal $l$, we define 
$$\hh{l}=\begin{cases}
not~a\hbox{ if $l$ is a literal of the form $\neg a$, where $a\in\sigma^{act}$}\\
\h{l}\hbox{ otherwise}
\end{cases}
$$

Program $simp_T(D)$ consists of the following rules:
\begin{enumerate}
	\item for every fluent atom $a$ the
	rules of the form~\eqref{eq:consistency}
	and~\eqref{eq:complete},
	\item for every static law~\eqref{eq:slaw} in $D$,
	$simp_T(D)$ contains
	the rules of the form
	\beq
	\h{l_0}(t) \ar\ not\ not\ \h{l_1}(t),\dots,\ not\ not\ \h{l_m}(t) 
	\eeq{eq:rulestatic2}
	for all $t=0,\dots,T$\footnote{\citeN{bab13} allow rules with arbitrary formulas in their bodies so that in place of~\eqref{eq:rulestatic2} they consider rule $\h{l_0}(t) \ar\ not\ not\ (\h{l_1}(t)\wedge \dots\wedge \h{l_m}(t))$. Yet, it is well known that such a rule is strongly equivalent to~\eqref{eq:rulestatic2}. Furthermore, more answer set solvers allow  rules of the form~\eqref{eq:rulestatic2} than more general rules considered in~\citep{bab13}.},
	\item for every dynamic law~\eqref{eq:dlaw} in $D$, 
	the rules
	\begin{align*}
	\h{l_0}(t+1) \ar\ not\ not\ \h{l_1}(t+1),\dots,\ not\ not\ \h{l_m}(t+1),\\  
	\hh{l_{m+1}}(t),\dots,\ \hh{l_n}(t), 
	\end{align*} 
	
	for all $t=0,\dots,T-1$,
	
	\item the rules
	\beq
	\ba{l}
	\{\t{a}(0)\}\\ 
	\{a(0)\}\\ 
	\ea
	\eeq{eq:fl2}
	for all fluent names  $a$ in $\sigma^{fl}$ and
	\item for every atom $a$ that is an action atom of the language of  $lp_T(D)$,
	the choice rules
	$
	\{a\}.$
	Here we note that the language \C assumes every action to be exogenous, whereas this is not the case in  $\C+$, where it has to be explicitly stated whether an action has this property. Thus, in~\citep{bab13} rules of this group only appear for the case of actions that have been stated exogenous. 
	
\end{enumerate}

The $simp_T$  translation of \C action description~\eqref{eq:cexample} consists of all rules of the form
$$
\ba{ll|ll}
1. &\bot\ar inWater(t),\t{inWater}(t)&2.
&wet(t)\ar\ not\ not\ {inWater}(t)\\

&\bot\ar\ not\ inWater(t),\ not\ \t{inWater}(t)&&\\

&\bot\ar wet(t),\t{wet}(t)&&\\
&\bot\ar\ not\  wet(t),\ not\ \t{wet}(t)&&\\
\hline
3.& inWater(t+1)\ar\ {putInWater}(t)&4.
&\{\t{inWater}(0)\}\\
&\{inWater(t+1)\}\ar\  {inWater}(t)&&\{inWater(0)\}\\

&\{\t{inWater}(t+1)\}\ar\  \t{inWater}(t)&&\{\t{wet}(0)\}\\
&\{wet(t+1)\}\ar\  {wet}(t)&&\{wet(0)\}\\
&\{\t{wet}(t+1)\}\ar\  \t{wet}(t)&&\\

\hline
5.&\{putInWater(t)\}\\
\ea
$$

\subsubsection{On the Relation Between Programs $lp_T$ and $simp_T$}\label{sec:relation}
Proposition~\ref{prop:main} stated in this section is the key result of this part of the paper. Its proof outlines the essential steps that we take in arguing that two logic programs $lp_T$ and $simp_T$ formalizing the action language \C are essentially the same.
The key claim of the proof is that logic program $lp_T(D)$ is a conservative extension of $simp_T(D)$. Here we only outline the proof whereas the Appendix of the paper provides a complete proof.

The argument of this claim requires some close attention to groups of rules in the $lp_T(D)$ program. In particular, we establish by means of weak natural deduction that
\begin{itemize}
	\item the rules in group 1 and 2 of $lp_T(D)$ are strongly equivalent to 
	the rules in group~1 and 2  of $simp_T(D)$ and 
	\item the rules in group 1 and 4 of $lp_T(D)$ are strongly equivalent to 
	the rules in group~1 and 4  of $simp_T(D)$. 
\end{itemize} 
Similarly, we  show that 
\begin{itemize}
	\item
	the rules in group 1 and 3 of $lp_T(D)$ are strongly equivalent to the rules in group 1 of $simp_T(D)$ and the rules structurally similar to rules in group 3 of  $simp_T(D)$ and yet not the same.
\end{itemize} 
These arguments allow us to construct a program $lp'_T(D)$, whose answer sets are the same as those of $lp_T(D)$. Program $lp'_T(D)$ is 
a conservative extension of $simp_T(D)$
due to explicit definition rewriting. Proposition~\ref{prop:ce} helps us to uncover this fact.

Recall that the language of $simp_T(D)$ includes the action atoms --- the action names of~$\sigma^{act}$ followed by $(t)$ where $t = 0,\dots T-1$. We denote the action atoms by $\sigma_T^{act}$. 

\begin{prop}\label{prop:main}
	For a set $X$ of atoms,
	$X$ is an answer set for $simp_T(D)$ if and only if the set 
	$X\cup \{\t{a}\mid a\in \sigma_T^{act}\setminus X\}$
	has the form  
	$$
	\Bigg[\bigcup_{t=0}^{T-1}{\{\h{l}(t)\mid l\in s_t\cup a_t\}}\Bigg]~\cup~\{\h{l}(T)\mid l\in s_T\}
	$$
	for some path $\langle s_0,a_0,s_1,\dots,s_{T-1},a_{T-1},s_T\rangle$ in the transition system described by~$D$.
	
\end{prop}

\subsubsection{Additional Concluding Remarks: An Interesting Byproduct}
Our work, which  illustrates that logic programs $lp_T(D)$ and $simp_T(D)$ are essentially the same, also uncovers a precise formal link between the action description languages $\C$ and $\C+$. The authors of $\C+$ claimed that $\C$ is an immediate predecessor of $\C+$. Yet,  to the best of our knowledge the exact formal link between the two languages has not been stated. Thus, earlier one could view $\C+$ as a generalization of~$\C$ only {\em informally} alluding to the fact that $\C+$ allows the same intuitive interpretation of syntactic expressions of $\C$, while generalizing these to allow multivalued fluents in place of Boolean ones. These languages share the same syntactic constructs 
such as, for example, a dynamic law 
of the form
$$ \caused f_0 \ifc f_1\wedge\cdots\wedge f_m \ \after\ a_{1}\wedge\cdots\wedge a_n. 
$$ 
that we intuitively read as after the concurrent execution of actions $a_{1}\dots a_n$ the fluent expression~$f_0$ holds in case if fluents expressions $f_1\dots f_m$ were the case at the time when aforementioned actions took place.   
Both languages provide interpretations for such expressions that meet our intuitions of this informal reading. Yet, if one studies the semantics of these languages it is not trivial to establish a specific formal link between them. For example, the semantics of $\C+$ relies on the concepts of causal theories~\citep{giu04}.  The semantics of $\C$ makes no reference to these theories. Here we recall the translations of \C and $\C+$ to logic programs, whose answer sets correspond to their key semantic objects. 
We then state the precise relation between the two by means of relating the relevant translations. In conclusion, $\C+$ can be viewed as a true generalization of the language~$\C$ to the case of multi-valued fluents.

\vspace{-1em}
\section{Programs with Variables}~\label{sec:foprogram}
We now proceed towards the second part of the paper devoted to programs with variables. We start by presenting its detailed outline, a new running example and then stating the preliminaries. We conclude with the formal statements on a number of rewriting techniques.

On the one hand, this part of the paper can be seen as a continuation of work by~\citeN{eit06}, where we consider common program rewritings using a more complex dialect of logic programs. On the other hand, this part of the paper grounds the concept of program synonymity studied by~\citeN{pea12} in a number of practical examples. Namely, we illustrate how formal results on strong equivalence developed earlier and in this work  help us  to construct precise claims about programs  in practice.



In this part of the paper, we systematically study some common rewritings  on first-order programs  utilized by ASP practitioners. 
We use a running example
to ground general theoretical presentation of this work into specific context. In particular,
 we consider two 
formalizations of a  planning module given in~\cite[Section 9]{gel14}: \begin{enumerate}[topsep=0pt]
	\itemsep0em 
	\item a \alchoice formalization that utilizes choice rules and aggregate expressions,
	\item a \dalchoice formalization that utilizes disjunctive rules.
\end{enumerate} 
Such a planning module is meant to be augmented with an ASP representation of a dynamic system description expressed in action language \al\footnote{It is due to remark that although~\citeN{gel14} use the word ``module'' when encoding a planning domain, they utilize this term only informally to refer to a collection of rules responsible for formalizing ``planning''.
}. \citeN{gel14} formally state in Proposition~9.1.1  that the answer sets of program  \dalchoice augmented with a given  system description  encode all the ``histories/plans'' of a specified length in the transition system captured by the system description. \textit{We note that no formal claim is provided for the \alchoice program.}
Although both \alchoice and \dalchoice programs {\em intuitively} encode the same knowledge the exact connection between them is not immediate. In fact, these programs 
\begin{itemize}
\item do not share the same signature, and
\item  use distinct syntactic constructs such as choice, disjunction, aggregates in the specification of a problem.
\end{itemize}
Here, we establish  
a one-to-one correspondence between the answer sets of these programs using
 their properties. 
Thus, the aforementioned formal claim about \dalchoice  translates into the same claim for   \alchoice.

Here we use a dialect of the ASP language called 
RASPL-1~\citep{lee08}. Notably, this language  combines  choice, aggregate, and disjunction constructs. Its semantics is given in terms of the SM operator, which exemplifies the approach to the semantics of first-order programs that bypasses grounding. Relying on SM-based semantics allows us to refer to  earlier work that study the formal properties of first-order programs~\citep{fer09,fer09a} using this operator. 
We state a sequence of formal results on programs rewritings and/or programs properties. Some of these results are geared by our running example and may not appear of great general interest. Yet, we view the proofs of these results as an interesting contribution as they showcase how arguments of correctness of rewritings can be constructed by the practitioners. 
Also, some discussed rewritings are well known and frequently used in practice. 
Often, their correctness is an immediate consequence of well known properties about logic programs (e.g., relation between intuitionistically provable first-order formulas and strongly equivalent programs viewed as such formulas).
Other discussed rewritings are far less straightforward and require elaborations on previous theoretical findings about the operator SM.
It is well known that propositional head-cycle-free disjunctive programs~\citep{ben94} can be rewritten to nondisjunctive programs by means of a simple syntactic transformation. Here we not only generalize this result to the case of first-order programs, but also illustrate that at times we can remove disjunction from parts of a program even though the program is not head-cycle-free. 
This result is relevant to local shifting and component-wise shifting discussed in~\citep{eit06} and~\citep{jan07}, respectively. 
We also generalize so called Completion Lemma and Lemma on Explicit Definitions
stated in~\citep{fer05,fer05b} for the case of propositional theories and propositional logic programs. These generalizations are applicable to first-order programs. We conclude by applying the Lemma on Explicit Definitions proved here to argue the correctness of program rewriting system~{\sc projector}~\cite{hip18}.

\subsection{Running  Example and Claims}
 
This section presents two ASP formalizations of a domain independent {\em planning module} given in \cite[Section 9]{gel14}. Such  planning module   is meant to be augmented with a logic program encoding a system description expressed in action language~\al that represents a domain of interest  (in Section 8 of their book~\cite{gel14}, Gelfond and Kahl present a sample Blocks World domain representation). 
Two formalizations of a planning module are stated here almost verbatim.
 Predicate names~$\occurs$ and~$\sthHpd$ intuitively stand for {\em occurs} and {\em something\_happend}, respectively. 
We  eliminate classical negation symbol by  
\begin{itemize}
\item utilizing auxiliary predicates 
 \noccurs in place of  $\neg \occurs$; and 
 \item  introducing rule
$\ar \occurs(A,I), \noccurs(A,I).$
\end{itemize} 
This is a standard practice and ASP systems  perform the same rewriting when processing classical negation symbol~$\neg$ occurring in programs (in other words, symbol $\neg$ is treated as  syntactic sugar). 

Let

\begin{tabular}{lll}
	$SG(I)$& abbreviate&~~
	$
	\step(I),\ not\ goal(I),\ I\neq n,
	$~~   \end{tabular}
where $n$ is some integer specifying  a limit on a length of an allowed plan.
The first formalization called \alchoice  follows:
\begin{flalign}
& success \ar goal(I),\ 
           \step(I). 
             \nonumber\\
&\ar not\ success.  \nonumber \\ 
&\ar \occurs(A,I), \noccurs(A,I)\label{eq:alchoice1}\\
&\noccurs(A,I) \ar \action(A),\ \step(I),\
                not\ \occurs(A,I)\label{eq:alchoice2}\\
 &    \{\occurs(A,I)\} \ar  \action(A),\ SG(I)  \label{eq:alchoice3}\\
&     \ar 2\leq\#count\{A: \occurs(A,I)\},\ SG(I).\label{eq:alchoice4}\\
&     \ar not\ 1\leq \#count\{A: \occurs(A,I)\}, \ SG(I)\label{eq:alchoice5}
\end{flalign}
One more remark is in order. In~\cite{gel14}, Gelfond and Kahl  list only a single rule 
$$
\ba{l}
1\{\occurs(A,I): \action(A)\}1 \ar  SG(I)  
\ea
$$
in place of  rules~(\ref{eq:alchoice3}-\ref{eq:alchoice5}). Note that this single rule is an  abbreviation for rules~(\ref{eq:alchoice3}-\ref{eq:alchoice5})~\citep{geb15}.


   The second formalization that we call a \dalchoice encoding   is obtained from \alchoice by  replacing  rules (\ref{eq:alchoice3}-\ref{eq:alchoice5})   with the following:
\begin{flalign}
&\occurs(A,I) \mid \noccurs(A,I) \ar \action(A),\ SG(I)\label{eq:dalchoice1}\\
&\ar \occurs(A,I),\ \occurs(A',I),\ \action(A),\ \action(A'),\ A \neq A'\label{eq:dalchoice2}\\
&\sthHpd(I) \ar \occurs(A,I)\label{eq:dalchoice3}\\
&\ar not\ \sthHpd(I),\ SG(I).\label{eq:dalchoice4}
\end{flalign}

It is important to note several  facts about the considered  planning module encodings. These planning modules are meant to be used with  logic programs  that capture  
\begin{itemize}
\item[(i)] a domain of interest originally stated as a system description in the action language \al; 
\item[(ii)] a specification of an initial configuration;
\item[(iii)] a specification of a goal configuration.
\end{itemize}
 The process of encoding (i-iii) as a logic program, which we call a \planinst encoding, follows a strict procedure. As a consequence, some important  properties hold about any \planinst. To state these it is 
convenient to recall the notion of a simple rule and define a ``terminal'' predicate.
 
 A {\em signature} is a set of function and 
 predicate symbols/constants.   A function symbol of arity~0 is an {\em object constant}. 
 A {\em term} is an object constant, an object variable, or an expression of the form $f(t_1,\dots,t_m)$, where $f$ is a function symbol of arity $m$ and each $t_i$ is a term. An {\em atom} is an expression of the form $p(t_1,\dots,t_n)$ or $t_1=t_2$, where $p$ is an $n$-ary predicate symbol and each $t_i$ is a term.
A {\em simple body} has the form 
  $$a_1,\dots,a_m,\ not\ a_{m+1},\dots,\ not\ a_{n},$$
  where~$a_i$ is an atom and $n\geq 0$.
  Expression $a_1,\dots,a_m$ forms the positive part of a body.
A {\em simple} rule has the form 
$$h_1\mid\cdots\mid h_k\ar \body$$
 or
$$
\{h_1\}\ar \body$$
where $h_i$ is an atom and $\body$ is a simple body. 
We now state a recursive definition of a terminal predicate with respect to a program.
Let~$i$ be a nonnegative integer.
A predicate that occurs only in rules whose  body  is empty is called {\em 0-terminal}.
We call a predicate {\em $i+1$-terminal} when it occurs only in the heads of simple rules (left hand side of an arrow), furthermore 
\begin{itemize}
\item in these rules  all predicates occurring in their positive parts of the bodies must be at most {\em $i$-terminal} and 
\item at least one of these rules is such that some predicate  occurring in its positive part of the body is {\em $i$-terminal} .
\end{itemize}  We call any $x$-terminal predicate {\em terminal}. For example, in program 
$$
\begin{array}{l}
 \block(b0).\   \block(b1).\\ 
\location(X) \ar \block(X).\ ~~~~~~ \location(table).
\end{array}
$$
$\block$ is a $0$-terminal predicate, $\location$ is a $1$-terminal predicate; and both predicates are terminal. 

We are now ready to state important {\em Facts} about any possible  \planinst and, consequently, about the considered planning modules 
\begin{enumerate}[topsep=0pt]
	\itemsep0em 
	\item\label{enum:fact1} Predicate $\occurs$ never occurs in the heads of rules in \planinst. 
	\item\label{enum:fact2} Predicates $\action$ and $\step$ are terminal in  \planinst as well as in \planinst augmented by either \alchoice or \dalchoice.
	\item\label{enum:fact2.5} By Facts~\ref{enum:fact1} and~\ref{enum:fact2}, predicate $\occurs$ is terminal in \planinst augmented by either \alchoice or \dalchoice.
	\item\label{enum:fact3} Predicate $\sthHpd$ never occurs in the heads of the rules in \planinst. 		
	\end{enumerate}
In the remainder of the paper we use considered  theoretical results to  illustrate  the following  {\observations}: 
\begin{enumerate}[topsep=0pt]
	\itemsep0em 
	
	\item\label{enum:observ4} In the presence of rule~\eqref{eq:alchoice2} it is safe to add  a rule
	\beq
	\noccurs(A,I) \ar not\ \occurs(A,I),\ \action(A),\ SG(I)
	\eeq{eq:disj2}
	into an arbitrary program.
By ``safe to add/replace'' we understand that the resulting program has the same answer sets as the original one.	
\item\label{enum:observ2}  It is safe to replace rule~\eqref{eq:alchoice4} with 
rule
\beq
\ar \occurs(A,I),\ \occurs(A',I),\ SG(I),\ A \neq A'
\eeq{eq:ruleSGI} within an arbitrary program.

\item\label{enum:observ1} In the presence of rules \eqref{eq:alchoice1} and  \eqref{eq:alchoice2}, it is safe to replace rule~\eqref{eq:alchoice3} with rule
\beq
\occurs(A,I)\ar not\ \noccurs(A,I),\ \action(A),\ SG(I)
\eeq{eq:disj1}
 within an arbitrary program. 

\item\label{enum:observ5} Given the syntactic features of the \alchoice encoding and any \planinst encoding,
 it is  safe to replace rule~\eqref{eq:alchoice3} with rule~\eqref{eq:dalchoice1}. The argument utilizes \observations~\ref{enum:observ4} and~\ref{enum:observ1}. 
 Fact~\ref{enum:fact3} forms an essential syntactic feature.

\item\label{enum:observ3} Given the syntactic features of the  \alchoice encoding  and any \planinst encoding, it is  safe to replace rule~\eqref{eq:alchoice4} with rule~\eqref{eq:dalchoice2}. The argument utilizes \observation~\ref{enum:observ2}, i.e., it is safe to replace rule~\eqref{eq:alchoice4} with 
rule
\eqref{eq:ruleSGI}. An essential syntactic feature relies on 
Fact~\ref{enum:fact1}, and the facts that (i) rule~\eqref{eq:alchoice3}
is  the only  one in \alchoice, where predicate $\occurs$ occurs in the head; and (ii)  rule~\eqref{eq:dalchoice2} differs from~\eqref{eq:ruleSGI} only in atoms that are part of the body of~\eqref{eq:alchoice3}.
 \item\label{enum:observ7} By Fact~\ref{enum:fact3}  and  
 the fact that \sthHpd does not occur in any other rule but~\eqref{eq:dalchoice4} in \dalchoice,
 the answer sets of the  program 
obtained by replacing rule~\eqref{eq:alchoice5}  with rules~\eqref{eq:dalchoice3} and~\eqref{eq:dalchoice4}
 are in one-to-one correspondence with the answer sets of the program \dalchoice extended with \planinst. 
\end{enumerate}

\smallskip \noindent{\bf Essential Equivalence Between Two Planning Modules:}
These {\observations} are sufficient to state that the answer sets of the \alchoice and 
\dalchoice programs (extended with any \planinst) are in one-to-one correspondence. We can capture the simple relation between the answer sets of these programs by
 observing that dropping the atoms whose predicate symbol is \sthHpd from an answer set of the \dalchoice program results in an answer set of the \alchoice program.

\vspace{-1em}
 \subsection{Preliminaries: RASPL-1 Logic Programs, Operator \SM, Strong Equivalence}\label{sec:sm}
 We now review a logic programming language  RASPL-1~\citep{lee08}. 
 This language is sufficient to capture choice, aggregate, and disjunction constructs (as used in  \alchoice and \dalchoice). There are distinct and not entirely compatible semantics for aggregate expressions in the literature. We refer the interested reader to the discussion by~\citeN{lee08} on the roots of semantics of aggregates considered in RASPL-1.

 An {\em aggregate expression} is an expression of the form 
 \beq
 b\leq \#count\{\vec{x}:L_1,\dots,L_k\}
 \eeq{eq:agg}
 ($k\geq 1$), where $b$ is a positive integer ({\em bound}), 
 $\vec{x}$ is a list of variables (possibly empty), and each 
 $L_i$ is an atom possibly preceded by {\em not}. We call variables in {\bf x} {\em aggregate} variables. This expression states that 
 there are at least~$b$ values of $\vec x$ such that conditions 
 $L_1,\dots, L_k$ hold.

 A {\em body} is an expression of the form
\beq e_{1} , \dots , e_m , not\  e_{m+1} , \dots 
 , not\  e_n
 \eeq{eq:body}
 ($n\geq m\geq 0$)  where  each $e_i$ is an aggregate expression or an atom.
 A rule is an expression of either of the forms
\begin{align}
&
 a_1  \mid \dots \mid a_l \ar  \body
 \label{eq:ruleraspl1} &\\
& \{a_1\} \ar \body
 \label{eq:ruleraspl2}&
 \end{align}
 ($l\geq 0$) where
 	each $a_i$ is an atom, and
 $\body$ is the body in the form~\eqref{eq:body}.
  When $l=0$, we identify the head of~\eqref{eq:ruleraspl1} with symbol~$\bot$ and call such a rule a {\em denial}. 
  When $l=1$, we call rule~\eqref{eq:ruleraspl1} a {\em defining} rule. 
  We call rule~\eqref{eq:ruleraspl2} a {\em choice} rule. 
 A {\em (logic) program} is a  set of {\em rules}. 
 An atom of the form $not\ t_1=t_2$ is abbreviated by $t_1\neq t_2$.


 
\vspace{-1em}
\subsubsection{Operator \SM}\label{sec:opsm}
 Typically, the semantics of logic programs with variables is given by  stating that these rules are  an abbreviation for a  possibly infinite set of  propositional rules. Then the semantics of propositional programs is considered. The {\SM} operator introduced by \citeN{fer09} 
 gives 
 a definition for the semantics of first-order programs
 bypassing  grounding. 
 It is an operator that takes a first-order sentence $F$ and a tuple ${\bf p}$ of predicate symbols  and produces the second order sentence that we denote by $\SM_{\bf p}[F]$.

 We now review the operator \SM.
 The symbols
 $ \bot,\land, \lor, \rar$, $\forall,$ and $\exists$ 
 are viewed as primitives. The formulas $\neg F$ 
 and $\top$ are 
 abbreviations for $F\rar\bot$
 and $\bot \rar 
 \bot$, respectively. 
 If $p$ and $q$ are predicate symbols of arity~$n$ then $p \leq q$ is an 
 abbreviation for the formula 
 $
 \forall {\bf x}(p({\bf x}) \rar q({\bf x})), 
 $
 where ${\bf x}$ is a tuple of variables of length~$n$. If ${\bf p}$ and ${\bf 
 	q}$ are tuples $p_1, \dots, p_n$ and $q_1, \dots, q_n$ of predicate symbols then 
 ${\bf p} \leq {\bf q}$ is an abbreviation for the conjunction 
 $
 (p_1 \leq q_1) \land \dots \land (p_n \leq q_n),
 $ 
 and 
 ${\bf p} < {\bf q}$ is an abbreviation for 
 $({\bf p} \leq {\bf q}) \land \neg ({\bf q} \leq {\bf p}).$ We apply the same 
 notation to tuples of predicate variables in second-order logic formulas.   
 If ${\bf p}$ is a tuple of predicate symbols $p_1, \dots, p_n$ (not including equality), and $F$ 
 is a first-order sentence 
 then \SM$_{\bf p}[F]$ denotes the second-order sentence 
 $$ F \land \neg \exists {\bf u}({\bf u} < {\bf p}) \land F^*({\bf u}), $$
 where ${\bf u}$ is a tuple of distinct predicate variables $u_1, \dots, u_n$, and $F^*({\bf u})$ is defined recursively:

 \begin{itemize}[topsep=0pt]
 	\setlength\itemsep{0em}
 	\item $p_i({\bf t})^*$ is $u_i({\bf t})$ for any tuple ${\bf t}$ of terms;
 	\item $F^*$ is $F$ for any atomic formula $F$ that does not contain members of 
 	$\bp$;\footnote{This includes equality statements and the formula $\bot$.} 
 	\item $(F \land G)^*$ is $F^* \land G^*$;
 	\item $(F \lor G)^*$ is $F^* \lor G^*$;
 	\item $(F \rar G)^*$ is $(F^* \rar G^*) \land (F \rar G) $;
 	\item $(\forall x F)^*$ is $\forall x F^*$;
 	\item $(\exists x F)^*$ is $\exists x F^*$.
 \end{itemize}
 Note that if ${\bf p}$ is the empty tuple then \SM$_{\bf p}[F]$ is equivalent to~$F$. For intuitions regarding the definition of the {\SM}  operator we direct the 
 reader to \cite[Sections  2.3, 2.4]{fer09}. 

By $\sigma(F)$ we denote the set of all 
function and 
predicate constants occurring in first-order formula $F$ (not including equality).
We will call this the {\em signature of $F$}. 
An interpretation~$I$ over $\sigma(F)$
is a {\em ${\bf p}$-stable model} of $F$ if it satisfies \SM$_{\bf p}[F]$, where {\bf p} is a tuple of predicates from $\sigma(F)$. 
We note that a ${\bf p}$-stable model of $F$ is also a model of $F$. 

By $\pred(F)$ we denote the set of all 
predicate constants (excluding equality) occurring in a formula $F$.
Let $F$ be 
 a first-order sentence that contains at least one object constant. 
 We call an Herbrand interpretation of $\sigma(F)$ that
 is a 
  $\pred(F)$-stable model of $F$ {\em an answer set}.\footnote{ An Herbrand interpretation of a signature $\sigma$ (containing at least one object constant) is such that its universe is the set of all ground terms of $\sigma$, and every ground term represents itself.   An Herbrand interpretation can be identified with the set of ground atoms (not containing equality) to which it assigns the value true.} 
 Theorem 1 from \citep{fer09} illustrates in which sense this definition can be seen as a generalization of a classical definition of an answer set (via grounding and reduct) for typical logic programs whose syntax is more restricted than syntax of programs considered here. 

  \vspace{-1em}
 \subsubsection{Semantics of Logic Programs}~\label{sec:sem}
 From this point on, we view logic program rules as alternative notation for particular 
 types of first-order 
 sentences.  
 We now define a procedure that turns every aggregate, every rule, and every program into a formula of first-order logic, called its {\em FOL representation}.
 First, we identify the logical connectives $\wedge$, $\vee$, and $\neg$  with their counterparts used in logic programs, namely, the comma, the disjunction symbol $\mid$, and connective $not$. This allows us to treat $L_1,\dots,L_k$ in~\eqref{eq:agg} as a conjunction of literals.

For an  aggregate expressions of the form
$$  b\leq \#count\{\vec{x}:F(\vec{x})\},$$
 its FOL representation   follows 
\beq
\exists \vec{x^1}\cdots  \vec{x^b} [\bigwedge_{1\leq i\leq b} F(\vec{x^i}) 
\wedge\bigwedge_{1\leq i<j\leq b} \neg(x^i=x^j)
]
\eeq{eq:formulaggexpr}
  where $\vec{x^1}\cdots  \vec{x^b}$ are lists of new variables of the same length as~$\vec{x}$.

The FOL representations of  logic rules  of the form~\eqref{eq:ruleraspl1} and~\eqref{eq:ruleraspl2} are formulas  
$$
\ba{l}

\widetilde{\forall} (\body\rightarrow a_1\vee\cdots\vee a_l)\hbox{~~~   and~~~~}
\widetilde{\forall}(\neg\neg a_1\wedge \body\rightarrow a_1),
\ea
$$ 
 where each aggregate expression in $\body$  is replaced by its FOL representation.  Symbol~$\widetilde{\forall}$ denotes universal closure.
 
  For example, 
  expression~$SG(I)$
  stands for formula
$
step(I)\wedge\neg goal(I)\wedge\neg I=n
$  
  and 
   rules (\ref{eq:alchoice3}) and (\ref{eq:alchoice5}) in the~\alchoice encoding 
  have the FOL representation: 
\begin{flalign}
&\widetilde{\forall}\big(\neg\neg \occurs(A,I)\wedge  SG(I)\wedge \action(A)\rightarrow \occurs(A,I)\big)\label{eq:folrep1}\\
&{\forall} I
\big( 
\neg \exists A[\occurs(A,I)] \wedge SG(I)
\rightarrow\bot\big)\label{eq:folrep3}
\end{flalign}
The FOL representation of rule (\ref{eq:alchoice4}) is the universal closure of the following implication
\begin{flalign}
&( 
\exists A  A' \big(\occurs(A,I)\wedge \occurs(A',I)\wedge \neg A=A'\big) \wedge SG(I)) 
\rightarrow\bot.\nonumber
\end{flalign}

We define a concept of an answer set for logic programs that   contain at least one object constant. This is inessential restriction as typical logic programs without object constants are in a sense trivial. In such programs,
whose semantics is given via grounding, rules with variables are eliminated during grounding. 
Let~$\Pi$ be a logic program with at least one object constant.
(In the sequel we often omit expression ``with at least one object constant''.) By~$\fol{\Pi}$ we denote its FOL representation. (Similarly, for a head $H$, a body $\body$, or a rule $R$, by $\fol{H}$, $\fol{\body}$, or $\fol{R}$ we denote their FOL representations.)
An {\em answer set} of $\Pi$ is an answer set of its  FOL representation  $\fol{\Pi}$.  In other words, an {\em answer set} of~$\Pi$ is an Herbrand interpretation of $\fol{\Pi}$ that is a $\pred(\fol{\Pi})$-stable model of $\fol{\Pi}$, i.e., a model of \beq
\SM_{\bf \pred(\fol{\Pi})}[\fol{\Pi}].
\eeq{eq:smpi} 
Sometimes, it is convenient to identify a logic program $\Pi$ with 
its semantic counterpart~\eqref{eq:smpi} so that formal results  stated in terms of {\SM} operator immediately translate into the results for logic programs.

  \vspace{-1em}
\subsubsection{Review: Strong Equivalence}\label{sec:strong}

We restate the definition of strong equivalence for first-order formulas given in~\citep{fer09} and recall some of its properties.
First-order formulas $F$ and $G$ are {\em strongly equivalent} if 
for any formula~$H$, any occurrence of~$F$ in~$H$, and any tuple~{$\bf p$} of 
distinct predicate constants,  \SM$_{\bf p}[H]$ is equivalent to 
\SM$_{\bf p}[H']$, where~$H'$ is obtained from $H$ by replacing~$F$ by~$G$. Trivially, any strongly equivalent formulas are such that their stable models coincide (relative to any tuple of predicate constants).
\citeN{lif07a} show that first-order 
formulas~$F$ and~$G$ are strongly equivalent if they are equivalent in \sqht logic --- an intermediate logic between classical and intuitionistic logics. Every formula provable using natural deduction, where 
the axiom of the law of the excluded middle ($F\vee\neg F$) is replaced by the weak law of the excluded middle ($\neg F\vee\neg \neg F$), is a theorem of \sqht. 

The definition of strong equivalence between first-order formulas  paves the way to  a definition of strong equivalence for logic programs. 
A logic program $\Pi_1$ is {\em strongly equivalent} to logic program~$\Pi_2$ when 
for any program $\Pi$, 
$$\SM_{\pi(\fol{\Pi\cup\Pi_1})}[\fol{\Pi\cup\Pi_1}]\hbox{ is equivalent to } \SM_{ \pi(\fol{\Pi\cup\Pi_2})}[\fol{\Pi\cup\Pi_2}].$$ 
It immediately follows that logic programs $\Pi_1$ and $\Pi_2$ are {\em strongly equivalent} if first-order formulas $\fol{\Pi_1}$ and $\fol{\Pi_2}$  are equivalent in logic of \sqht.

We now review an important result about properties of denials. 
\begin{theorem}[Theorem 3~\citep{fer09}]\label{thm:constraints}
	For any first-order formulas $F$ and $G$ and arbitrary tuple~${\bf p}$ of predicate constants, $\SM_{\bf p}[F\wedge\neg G]$ is equivalent to 
	$\SM_{\bf p}[F]\wedge\neg G.$
\end{theorem}
As a consequence, {\bf p}-stable models of   $F\wedge\neg G$ can be characterized as the {\bf p}-stable models of   $F$ that satisfy first-order logic formula $\neg G$.
Consider any denial 
$
\ar \body.
$
Its  FOL representation has the form $\widetilde{\forall}(\body\rar \bot)$ that is intuitionistically equivalent to formula $\neg \widetilde{\exists} \body$.
Thus, Theorem~\ref{thm:constraints} tells us that given any denial of a program it is safe to compute answer sets of a program without this denial and a posteriori verify that the FOL representation of a denial is satisfied. 
\begin{corollary}\label{cor:constraints}
Two denials are strongly equivalent if
 their FOL representations are classically equivalent.
\end{corollary}
This corollary is also an immediate consequence of 
 the Replacement Theorem for intuitionistic logic for first-order formulas~\cite{min00} stated below.

\begin{prop}[Replacement Theorem II \cite{min00}, Section 13.1]
	If $F$ is a first-order formula  containing a subformula $G$ and~$F'$ is the result of replacing that subformula by $G'$ then $\widetilde\forall(G\lrar G')$ intuitionistically implies $F\lrar F'$. 
\end{prop}

\vspace{-1em}
 \subsection{Rewritings}\label{sec:rewr}

\vspace{-.5em}
\subsubsection{Rewritings via Pure Strong Equivalence}
Strong equivalence 
can be used to argue the 
correctness of some program rewritings practiced by ASP software engineers. Here we state several theorems about strong equivalence between programs. \observations~\ref{enum:observ4},~\ref{enum:observ2}, 
and~\ref{enum:observ1} are consequences of these results.

We say that body $\body$ subsumes  body $\body'$ when $\body'$ has the form $\body,\body''$ (note that an order of expressions in a body is immaterial) . We say that a rule~$R$ subsumes rule~$R'$ when heads of $R$ and $R'$ coincide while body of $R$ subsumes body of $R'$.
For example, rule~\eqref{eq:alchoice2} 
 subsumes rule~\eqref{eq:disj2}.

{\bf Subsumption Rewriting:} 
Let ${\bf R'}$ denote a set of rules subsumed by rule $R$. It is easy to see that formulas $\fol{R}$ and  $\fol{R}\wedge \fol{\bf R'}$ are intuitionistically equivalent.
 Thus, program composed of rule $R$ and program $\{R\}\cup {\bf R'}$ are strongly equivalent.
It immediately follows that \observation~\ref{enum:observ4} holds. Indeed, rule~\eqref{eq:alchoice2} is strongly equivalent to the set of rules composed of itself and~\eqref{eq:disj2}. Indeed, rule~\eqref{eq:alchoice2} 
subsumes rule~\eqref{eq:disj2}. 

{\bf Removing Aggregates:} 
The following theorem is an immediate consequence of the Replacement Theorem II.
\begin{theorem}\label{thm:Aggregates}
	Program
	\beq
H	\ar  b\leq \#count\{\vec{x}:F(\vec{x})\},\ G
	\eeq{eq:ruleag}
	is strongly equivalent to
	program
	\beq
H \ar \bigcomma_{1\leq i\leq b} F(\vec{x^i}) 
	\bigcomma_{1\leq i<j\leq b} x^i\neq x^j,\ G
	\eeq{eq:rulenoag}
	where $G$ and $H$ have no occurrences of variables in $\vec{x^i}$ $(1\leq i\leq b)$.
\end{theorem}

Theorem~\ref{thm:Aggregates} 
 shows us that
\observation~\ref{enum:observ2} is a special case of a more general fact.
Indeed, take rules~\eqref{eq:alchoice4} and~\eqref{eq:ruleSGI} to be the instances of  rules~\eqref{eq:ruleag} and~\eqref{eq:rulenoag} respectively.

We note that  the Replacement Theorem II also allows us  to immediately conclude the following. 
\begin{corollary}\label{cor:gen}
	Program~~
$H	\ar  G$~~
	is strongly equivalent to
	program~~
$ 
H \ar G' 
$~~
when
~~$\widetilde\forall(\fol{G}\lrar \fol{G'})$.
\end{corollary}
Corollary~\ref{cor:gen}  equips us with a general semantic condition that can be utilized in proving the syntactic properties of programs in spirit of Theorem~\ref{thm:Aggregates}. 

{\bf Replacing Choice Rule by Defining Rule:}
Theorem~\ref{thm:choice} shows us that
\observation~\ref{enum:observ1} is  an instance of a more general fact.
\begin{theorem}\label{thm:choice}
	Program
	\begin{align}
	&\ar p(\vec{x}),\ q(\vec{x})\label{l1}&\\
	& q(\vec{x})\ar not\ p(\vec{x}), F_1\label{l2}&\\
	&\{p(\vec{x})\}\ar F_1,\ F_2\label{l3}&
	\end{align}
	is strongly equivalent to program composed of rules~\eqref{l1},~\eqref{l2}  and rule 
	\beq
	p(\vec{x})\ar not\ q(\vec{x}),\ F_1,\  F_2,
	\eeq{eq:last}
	where $F_1$ and $F_2$ are the expressions of the form~\eqref{eq:body}.
\end{theorem}
   

To illustrate the correctness of \observation~\ref{enum:observ1} by Theorem~\ref{thm:choice}:
(i)
take rules
\eqref{eq:alchoice1},  \eqref{eq:alchoice2},  \eqref{eq:alchoice3} be the instances of rules (\ref{l1}), (\ref{l2}), (\ref{l3}) respectively,  and (ii) rule~\eqref{eq:disj1} be the instance of rule \eqref{eq:last}.

\vspace{-1em}
\subsubsection{Useful Rewritings using Structure}
In this subsection, we study rewritings on a program that rely on its structure.
We review the concept of a  dependency graph used in posing structural conditions on rewritings.

\paragraph{Review: Predicate Dependency Graph}
 We present the concept of the predicate dependency graph of a formula following the lines of~\citep{fer09a}. An occurrence of a predicate constant, or any other subexpression,  in a formula is called {\em positive} if the number of implications containing that occurrence in the antecedent is even, and {\em strictly positive} if that number is 0. We say that an occurrence of a predicate constant is {\em negated} if it belongs to a subformula of the form $\neg F$ (an abbreviation for $F\rar\bot$), and {\em nonnegated} otherwise.
 
  For instance, in formula~\eqref{eq:folrep1}, predicate constant $\occurs$ has a strictly positive occurrence  in the consequence of the implication;
  whereas the same symbol $\occurs$ has a negated positive occurrence in the antecedent  
  \beq
  \neg\neg \occurs(A,I)\wedge \step(I)\wedge\neg goal(I)\wedge\neg I= n\wedge \action(A)
  \eeq{eq:antec}
  of~\eqref{eq:folrep1}. 
  Predicate symbol $\action$ has a strictly positive non-negated occurrence in~\eqref{eq:antec}. The occurrence of predicate symbol $goal$ is negated and not positive in~\eqref{eq:antec}. 
  The occurrence of predicate symbol $goal$ is negated and  positive in~\eqref{eq:folrep1}.
 
 An {\em FOL rule} of a first-order formula $F$ is a strictly positive occurrence of an implication in $F$. For instance, in a conjunction of two  formulas~(\ref{eq:folrep1}) and  (\ref{eq:folrep3}) the 
 FOL rules are as follows
\begin{align}
&\neg\neg \occurs(A,I)\wedge  SG(I)\wedge \action(A)\rightarrow \occurs(A,I)&\label{fl:1}\\
&\neg \exists A[\occurs(A,I)] \wedge SG(I)
\rightarrow\bot.&\label{fl:2}
\end{align}
 
 For any first-order formula $F$, the {\em (predicate) dependency graph} of $F$ relative to the tuple {\bf p} of predicate symbols (excluding $=$) is the directed graph that
 	(i) has all predicates in~{\bf p} as its vertices, and
 	(ii) has an edge from $p$ to $q$ if for some FOL rule  $G\rar H$ of $F$
 	\begin{itemize}[topsep=0pt]
 		\itemsep0em 
 		\item $p$ has a strictly positive occurrence in $H$, and 
 		\item $q$ has a  positive nonnegated occurrence in $G$.
 	\end{itemize}
 We denote such a graph by $DG_{\bf p}[F]$.
 For instance, 
 Figure~\ref{fig:dpgpart} presents 
 the dependence graph of a conjunction of formulas~(\ref{eq:folrep1}) and (\ref{eq:folrep3}) relative to all its predicate symbols. It  contains four vertices, namely, $\occurs$, $\action$, $\step$, and $goal$, and  two edges: one from vertex $\occurs$ to vertex $\action$ and the other one from  $\occurs$ to $\step$.
 Indeed, consider the only two FOL rules~\eqref{fl:1} and~\eqref{fl:2} stemming from  this conjunction. Predicate constant $\occurs$ has a strictly positive occurrence in the consequent $\occurs(A,I)$ of the implication~\eqref{fl:1}, whereas  $\action$ and $\step$ are the only predicate constants in the antecedent $\neg\neg \occurs(A,I)\wedge  SG(I)\wedge \action(A)$  of~\eqref{fl:1} that have positive and nonnegated occurrence in this antecedent.  It is easy to see that a FOL rule of the form $G\rar\bot$, e.g., FOL rule \eqref{fl:2}, does not contribute edges to any dependency graph.
  \begin{figure}
  	
  		
  		\centering
  	
  		\begin{tikzpicture} [shorten >=.3pt, auto,scale=0.5]
  		\node[place,scale=0.8] (a)   {{\it \action}};
  		\node[place,scale=0.8] (occurs) [left of=a]  {{\it \occurs}};
  		    \node[place,scale=0.8] (step) [left of=occurs] {{\it \step}};
  		\node[place,scale=0.8] (goal) [left of=step]  {{\it goal}};
  		\path[->]
  		(occurs)  edge[->] node[above] {} (a)
  		(occurs)  edge[->] node[above] {} (step)
  		
  		;

  		\end{tikzpicture}
  		\normalsize
  		\caption{The predicate dependency graph 
 of a conjunction of formulas (\ref{eq:folrep1}) and (\ref{eq:folrep3}).} 
  		\label{fig:dpgpart}
  	\normalsize
  \end{figure}
 
 For any logic program $\Pi$, the {\em dependency graph} of $\Pi$, denoted $DG[\Pi]$,
 is a directed graph of~$\fol{\Pi}$ relative to 
 the predicates occurring in $\Pi$. For example, let $\Pi$ be composed of two rules (\ref{eq:alchoice3}) and~(\ref{eq:alchoice5}). The  conjunction of   formulas~(\ref{eq:folrep1}) and~(\ref{eq:folrep3}) forms its FOL representation. 
 Thus, Figure~\ref{fig:dpgpart} captures its dependency graph $DG[\Pi]$.

\paragraph{Shifting}
We call a logic program {\em disjunctive} if all its rules have the form~\eqref{eq:ruleraspl1}, where  $\body$ only contains atoms possibly preceded by $not$. We say that a disjunctive program is {\em normal} when it does  not contain disjunction  connective $\mid $. 
\citeN{gel91a} defined a mapping from a propositional disjunctive program $\Pi$ to a  propositional normal program 
by replacing each rule~\eqref{eq:ruleraspl1}  with $l>1$ in $\Pi$ by  $l$ new rules
	$$
a_i  \ar  \body,\ not\ a_1,   \dots  not\ a_{i-1}, not\ a_{i+1},\dots not\ a_l.
$$
 They showed that every answer set of the constructed program 
is also an answer set of $\Pi$. Although
the converse does not hold in general, \citeN{ben94} 
 showed that
the converse holds if $\Pi$  is ``head-cycle-free''. \citeN{lin04a}  illustrated how this property holds about programs with nested expressions that  capture choice rules, for instance. 
Here we  generalize these findings further. First, we show that shifting is applicable to first-order programs (which also may contain choice rules and aggregates in addition to disjunction).  Second, we illustrate that under certain syntactic/structural conditions on a program we may apply shifting  ``locally'' to some rules with disjunction and not others.


For an atom $a$, by $a^0$ we denote its predicate constant. For example $\occurs(A,I)^0=\occurs$.
Let~$R$ be a rule of the form~\eqref{eq:ruleraspl1} with~$l>1$. 
By $\shift_{\bf p}(R)$ (where ${\mathbf p}$ is a tuple of distinct predicates) we denote the rule 
\beq
\bigmid_{\hbox{$1\leq i\leq l$, $a_i^0\in \bf p$}}
a_i  \ar  \body, \bigcomma_{\hbox{$1\leq j\leq l$, $a_j^0\not\in\bf p$}} not\ {a_j} .
\eeq{eq:ruleraspl1shift}
Let ${\mathcal{P}^R}$ be a partition of the set composed of the distinct predicate symbols  occurring in the head of rule $R$. 
By $\shift_{\mathcal{P}^R}(R)$ we denote the set of rules composed of rule $\shift_{\bf p}(R)$ for every  member~{\bf p} of the partition $\mathcal{P}^R$ (order of the elements in {\bf p} is immaterial).

For instance, if $R_1$ denotes a rule
\beq
a\mid b\mid c\mid d \mid e(1) \ar
\eeq{eq:longdisjunction}
then ${\mathcal P_1}^{R_1}=\{\{a,b\},\{c,d,e\}\}$ and
${\mathcal P_2}^{R_1}=\{\{a,b\},\{c\},\{d,e\}\}$
form  sample partitions of the described kind. 
Set $\shift_{\mathcal{P}_1^{R_1}}(R_1)$ consists of rules
$$
\ba{l}
a\mid b\ar not\  c,\ not\  d,\ not\   e(1) \\
c\mid d\mid e(1)\ar not\  a,\ not\  b,
\ea
$$
whereas set $\shift_{\mathcal{P}_2^{R_1}}({R_1})$ consists of rules
$$
\ba{l}
a\mid b\ar not\  c,\ not\  d,\ not\   e(1) \\
c\ar  not\  a,\ not\  b,\ not\ d,\ not\  e(1)\\
d\mid e(1)\ar not\  a,\ not\  b,\ not\ c. 
\ea
$$

\begin{theorem}\label{thm:shift2}
	Let $\Pi$ be  a logic program, 
	$\bf R$ be a set of rules  in~$\Pi$ of the form~\eqref{eq:ruleraspl1} with $l>1$, and $C$ be the set of strongly connected components in the dependency graph of $\Pi$.
	A program constructed from~$\Pi$ by replacing each rule $R\in \bf R$ with  $\shift_{\mathcal{P}^R}(R)$ has the same answer sets as~$\Pi$
	if  any  partition~$\mathcal{P}^R$ is such that there are no two distinct members ${\bf p}_1$ and ${\bf p}_2$ in $\mathcal{P}^R$ so that for some 
	strongly connected component $c$ in $C$, $c\cap \bf p_1\neq \emptyset$ and 
	$c\cap \bf p_2\neq \emptyset$.\footnote{The statement of this theorem  was suggested by Pedro Cabalar and Jorge Fandinno in personal communication on January 17, 2019.}
\end{theorem}

Consider a sample program~$\Pi_{samp}$ composed of rule~\eqref{eq:longdisjunction}, which we denote as $R_1$, 
and rules
\beq 
\ba{l}
a\ar b\\
b\ar a.
\ea
\eeq{eq:defsample}
The strongly connected components of  program $\Pi_{samp}$ are $\{\{a,b\},\{c\},\{d\},\{e(1)\}\}$.
Theorem~\ref{thm:shift2} tells us, for instance, that the answer sets of program~$\Pi_{samp}$ coincide with the answer sets of two distinct programs:
\begin{enumerate}
	\item a program composed of rules $\shift_{\mathcal{P}_1^{R_1}}(R_1)$ and rules~\eqref{eq:defsample}; 
	\item a program composed of rules $\shift_{\mathcal{P}_2^{R_1}}(R_1)$ and rules~\eqref{eq:defsample}.
	\end{enumerate}

\smallskip
We now use Theorem~\ref{thm:shift2} to argue the correctness of  \observation~\ref{enum:observ5}.
Let \hbox{\alchoice}$'$ denote a  program constructed from the \alchoice encoding by replacing
\eqref{eq:alchoice3} with
\eqref{eq:dalchoice1}.
Let \hbox{\alchoice}$''$ denote a  program
constructed from the \hbox{\alchoice}, by (i) replacing \eqref{eq:alchoice3} with~\eqref{eq:disj1}  and (ii) adding rule~\eqref{eq:disj2}.
Theorem~\ref{thm:shift2} tells us that programs 
  \hbox{\alchoice}$'$ and  \hbox{\alchoice}$''$ have the same answer sets.
  Indeed, \begin{enumerate}[topsep=0pt] 
   		\itemsep0em
  	\item take~${\bf R}$ to consist of rule~\eqref{eq:dalchoice1} and
  	\item recall  Facts~\ref{enum:fact1},~\ref{enum:fact2}, and~\ref{enum:fact2.5}. Given any  \planinst intended to use with \alchoice a program 
  obtained from the union of \planinst and  \hbox{\alchoice}$'$
  is such that  $\occurs$ is terminal. 
  It is easy to see that any terminal predicate in a program occurs only in the singleton strongly connected components of a program dependency graph.
  \end{enumerate}
   Due to \observations~\ref{enum:observ4} and~\ref{enum:observ1}, the \hbox{\alchoice} encoding has the same answer sets as \hbox{\alchoice}$''$ and consequently the same answer sets as \hbox{\alchoice}$'$. This argument accounts for the proof of \observation~\ref{enum:observ5}.
  
\medskip


\paragraph{Completion}

We now proceed at stating formal results about  first-order formulas and their stable models. The fact that we identify logic programs with their FOL representations translates these results to the case of the RASPL-1 programs.


About a
first-order formula $F$ we  say that it is in
{\em  Clark normal form}~\citep{fer09}
relative to
the tuple/set
{\bf p} of  predicate symbols if it is a conjunction of formulas of the
form
\beq 
\forall \vec{x} (G\rar p(\vec{x}))
\eeq{eq:compformula}
one for each predicate $p\in {\bf p}$, where $\vec{x}$ is a tuple of distinct object variables. 
We refer the reader to Section 6.1 in~\citep{fer09} for the description of  the intuitionistically equivalent transformations that can convert a first-order formula, which is a FOL representation for a RASPL-1 program (without disjunction and denials),  into Clark normal form.

The {\em completion} of a formula  $F$ in Clark normal form relative to predicate symbols~{\bf p}, denoted by $Comp_{\bf p}[F]$, 
is obtained from $F$ by replacing each conjunctive term of the form~\eqref{eq:compformula}
with 
$$\forall \vec{x} (G\lrar p(\vec{x})).$$

We now review an important result about properties of completion. 
\begin{theorem}[Theorem 10~\citep{fer09}]\label{thm:comp0}
	For any formula $F$ in Clark normal form and arbitrary tuple~${\bf p}$ of predicate constants, formula $$\SM_{\bf p}[F]\rar Comp_{\bf p}[F]$$ is logically valid.
\end{theorem}

The following Corollary is an immediate consequence of  this theorem, Theorem~\ref{thm:constraints}, and the fact that formula of the form $\widetilde{\forall}(\body\rar \bot)$  is intuitionistically equivalent to formula $\neg \widetilde{\exists} \body$.
\begin{corollary}\label{cor:comp}
For any formula $G\wedge H$ such that (i)
formula $G$ is in Clark normal form relative to {\bf p} and $H$ is a conjunction of formulas of the form $\widetilde{\forall}(K\rar\bot)$,
 the implication
 $$
 \SM_{\bf p}[G\wedge H]\rar Comp_{\bf p}[G]\wedge H
 $$
 is logically valid.
\end{corollary}

To illustrate the utility of this result we now construct an argument for the correctness of \observation~\ref{enum:observ3}. This argument finds one more formal result of use:
\begin{theorem}\label{prop:equivelast}
For a program $\Pi$, a first-order formula $F$ such that every answer set of~$\Pi$  satisfies $F$, and any two denials $R$ and $R'$ such that $F\rar(\fol{R}\leftrightarrow\fol{R'})$,  
the answer sets of programs $\Pi\cup \{R\}$ and $\Pi\cup \{R'\}$   coincide.
\end{theorem}
Theorem~\ref{thm:constraints} provides grounds for a straightforward argument for this statement.

Consider 
the \alchoice encoding without denial~\eqref{eq:alchoice4}   extended with any \planinst.
We can partition it into two parts: one that contains the denials, denoted by~$\Pi_H$, and the remainder, denoted by~$\Pi_G$.
Recall Fact~\ref{enum:fact1}.
Following the steps described by \citeN[Section 6.1]{fer09},  formula $\fol{\Pi_G}$ turned into  Clark normal form relative to the predicate symbols occurring in $\Pi_H\cup \Pi_G$   contains implication~\eqref{eq:folrep1}. The completion of this formula contains equivalence 
\beq
\widetilde{\forall}\big(\neg\neg \occurs(A,I)\wedge  SG(I)\wedge \action(A)\leftrightarrow \occurs(A,I)\big).
\eeq{eq:streqcond}
By Corollary~\ref{cor:comp} it follows that any answer set of $\Pi_H\cup \Pi_G$ 
satisfies formula~\eqref{eq:streqcond}. 
It is easy to see that an interpretation satisfies~\eqref{eq:streqcond} and the FOL representation of~\eqref{eq:ruleSGI} 
if and only if it satisfies~\eqref{eq:streqcond} and the FOL representation of denial~\eqref{eq:dalchoice2}.  Thus, by Theorem~\ref{prop:equivelast} program \hbox{$\Pi_H\cup \Pi_G$} extended with~\eqref{eq:ruleSGI} and program \hbox{$\Pi_H\cup \Pi_G$} extended with~\eqref{eq:dalchoice2} have the same answer sets. Recall \observation~\ref{enum:observ2} claiming that  it is safe to replace denial~\eqref{eq:alchoice4} with 
denial~\eqref{eq:ruleSGI} within an arbitrary program. It follows that
program $\Pi_H\cup \Pi_G$ extended with~\eqref{eq:dalchoice2} have the same answer sets $\Pi_H\cup \Pi_G$ extended with~\eqref{eq:alchoice4}.
 This concludes the argument for the claim of \observation~\ref{enum:observ3}.

We now state the main formal results of the second part of the paper. The Completion Lemma for first-order programs stated next is essential in proving the Lemma on Explicit Definitions for first-order programs. 
\observation~\ref{enum:observ7} follows immediately from the latter lemma.

\begin{theorem}[Completion Lemma]\label{thm:complemma}
	Let $F$ be a first-order formula  and {\bf q} be a set of predicate constants that do not have positive, nonnegated occurrences in any FOL rule of~$F$. 
	Let $\bf p$ be a set of predicates in $F$ disjoint from {\bf q}. 
	Let $D$ be a formula in Clark normal form relative to {\bf q} so that in every conjunctive term~\eqref{eq:compformula} of~$D$ 
		no  occurrence of an element in ${\bf q}$ occurs 
		in $G$ as positive and nonnegated.
Formula	 
\beq
\hbox{$\SM_{\bf pq}[F\wedge D]$}
\eeq{eq:goal}
is equivalent to formulas
\begin{flalign}
&\SM_{\bf pq}[F\wedge D]\wedge Comp[D],\label{goal1}\\
&\SM_{\bf p}[F]\wedge Comp[D], \hbox{ and }\label{goal2}\\
&\SM_{\bf pq}[F\wedge \bigwedge_{q\in\{\bf q\}} {\forall \vec{x} 
\big( \neg \neg q(\vec x)\rar 
q(\vec{x})\big)}]\wedge Comp[D].\label{goal3}
\end{flalign}

\end{theorem}
This result tells us that {\bf pq}-stable models of $F\wedge D$ are such that they satisfy the classical first-order formula $Comp[D]$. These models also can be characterized as (i) the {\bf p}-stable models of~$F$ that satisfy $Comp[D]$, and 
(ii) the {\bf pq}-stable models of~\hbox{$F$} extended with the counterpart of choice rules for member of {\bf q} that satisfy $Comp[D]$.

For an interpretation~$I$ over signature $\Sigma$, 
by $I_{|\sigma}$ we denote the 
interpretation over $\sigma\subseteq\Sigma$ constructed from $I$ so that  every function or 
predicate symbol  in $\sigma$ is assigned the same value in both~$I$ and~$I_{|\sigma}$. 
We call formula $G$ in~\eqref{eq:compformula} a {\em definition} of~$p(\vec{x})$.
\begin{theorem}[Lemma on Explicit Definitions]\label{thm:complemma2}
	Let $F$ be a first-order formula, {\bf q} be a set of predicate constants that do not occur in  $F$, and  {\bf p} be an arbitrary set of predicate constants in  $F$. 
	Let $D$ be a formula in Clark normal form relative to {\bf q} so that in every conjunctive term~\eqref{eq:compformula} of $D$ 
	there is no  occurrence of an element in ${\bf q}$ in $G$.
	Then 
	\begin{itemize}
		\item[i] $M\mapsto M_{|\sigma(F)}$  is a 1-1 correspondence between the models of $\SM_{\bf pq}[F\wedge D]$ 
		and the models $\SM_{\bf p}[F]$, and
		\item[ii]  $\SM_{\bf pq}[F\wedge D]$ and $\SM_{\bf pq}[F^{\bf q}\wedge D]$ are equivalent, where we understand $F^{\bf q}$ as a formula obtained from~$F$ by replacing occurrences of the definitions of 
		$q(\vec{x})$ in~$D$ with~$q(\vec{x})$.
		\item[iii]  $\SM_{\bf pq}[F\wedge D]$ and $\SM_{\bf pq}[F'^{\bf q}\wedge D]$ are equivalent, where we understand $F'^{\bf q}$ as a formula obtained from~$F$ by replacing occurrences of any subformula of the definitions of 
		$q(\vec{x})$ in~$D$ with~$q(\vec{x})$.
	\end{itemize}

\end{theorem}


It is easy to see that
	the program composed of the single rule
	$$
	p(\vec{y})\ar  1\leq \#count\{\vec{x}:F(\vec{x},\vec{y})\}
	$$ 
	and	the program
		$p(\vec{y})\ar F(\vec{x},\vec{y})$
are strongly equivalent.
Thus, we can identify 
 rule \eqref{eq:dalchoice3} in the \dalchoice encoding with
the rule
\beq
	\sthHpd(I)\ar 1\leq \#count\{A: \occurs(A,I)\}.
	\eeq{eq:sth2}
Using this fact and Theorem~\ref{thm:complemma2} allows us to support \observation~\ref{enum:observ7}. Take $F$ to be the FOL representation of \alchoice encoding extended with any \planinst and $D$ be the FOL representation of~\eqref{eq:sth2}, {\bf q} be composed of a single predicate \sthHpd and   {\bf p} be composed of all the predicates in \alchoice and \planinst.

\subsection{Projection}\label{sec:proj} 

\citeN{har16} considered a rewriting technique called projection. 
We start by reviewing their results. We then illustrate how the theory developed here is applicable in their settings. Furthermore, it allows us to  generalize their results to more complex programs.
In addition, our results give us a proof of correctness for system {\sc projector}~\cite{hip18} that implements so called $\alpha$ and $\beta$-projections.

\citeN{har16} considered programs to be  
first-order sentence formed as a conjunction of 
formulas of the form $$
\widetilde{\forall} (a_{k+1} \land \dots \land a_l \land \neg a_{l+1} \land \dots 
\land \neg a_m \land \neg \neg a_{m+1} \land \dots \land \neg \neg a_n \rar a_1 
\lor \dots \lor a_k). 
$$
It is easy to see that the FOL-representation of RASPL-1 rule without aggregate expressions comply with this form. We will now generalize  the main result by \citeN{har16} to the case of  RASPL-1 programs. 

Recall how in Section~\ref{sec:sem}  we identify the logical connective $\neg$  with its counterpart used in logic programs, namely,  $not$. This allows us to call an expression $not\ a$, where $a$ is an atom,  a literal. To simplify the presentation of rewriting in this section we will treat $L_1,\dots,L_k$ in~\eqref{eq:agg} as a set of literals. We will also identify
body~\eqref{eq:body} with the set $\{e_{1} , \dots , e_m , not\  e_{m+1}, not\  e_n\}$ of its elements. 

Let~$R$ be a RASPL-1 rule in a program $\Pi$, and let 
{\bf x} be a non-empty tuple of variables occurring only in 
body of~$R$ outside of any aggregate expression. By $\alpha({\bf x,y})$ we denote a set of literals in the body of $R$ so that it includes all literals in the body of $R$ that contain at least one variable from~{\bf x}. Tuple {\bf y} denotes 
all the variables occurring in the literals of   $\alpha({\bf x,y})$
different from~{\bf x}. 
By $\alpha'$ we denote any subset of $\alpha({\bf x,y})$ whose literals 
do not contain any variables occurring in ${\bf x}$. By $\body$ and $H$ we denote the body and the head of~$R$ respectively. Let $u$ be a 
predicate symbol that does not occur in $\Pi$. Then a 
\emph{result of projecting variables~${\bf x}$ out of~$R$ using predicate symbol $u$} consists of the 
following two rules  
$$
\ba{l}
H\ar(\body\setminus\alpha({\bf x},{\bf y}))\cup\alpha'\cup\{u(\bf y)\}\\
u(\bf y)\ar\alpha({\bf x},{\bf y}). 
\ea
$$
For example, one possible  result of projecting $Y$ out of 
\beq
s(X,Z)\ar p(Z), q(X,Y), r(X,Y), t(X)
\eeq{eq:s1def}
using predicate symbol $u$ is 
\begin{flalign}
&s(X,Z)\ar u(X), p(Z), t(X)\label{eq:al1}\\
&u(X)\ar q(X,Y),r(X,Y).\label{eq:al2}
\end{flalign}
 Another possible  result of projecting $Y$ out of rule~\eqref{eq:s1def}
using predicate symbol $u$ consists of rule~\eqref{eq:al1}
and rule
\beq
\ba{l}
u(X)\ar q(X,Y),r(X,Y),t(X).
\ea
\eeq{eq:sprime2def}
Yet, another possible  result of projecting $Y$ out of rule~\eqref{eq:s1def}
 using predicate symbol $u$ consists of rule
\beq
s(X,Z)\ar u(X), p(Z)
\eeq{eq:al3}
 and rule~\eqref{eq:sprime2def}.

We are now ready to state a formal result about projecting that is a generalization of Theorem~6 in~\cite{har16}. 


\begin{theorem}\label{thm:projection} Let~$R$ be a RASPL-1 rule in a program $\Pi$, and let 
	{\bf x} be a non-empty tuple of variables occurring only in 
	body of~$R$ outside of any aggregate expression and not in the head. 	 
	If~$\Pi'$ is constructed from $\Pi$ by replacing~$R$ in~$\Pi$ with a result of 
	projecting variables~${\bf x}$ out of~$R$ using a predicate symbol~$u$ 
	that is not in the signature of $\Pi$,
	then  $M\mapsto M_{|\sigma(\fol{\Pi})}$  is a 1-1 correspondence between the models of 
	SM$_{{\bf p}, u}[\fol{\Pi'}]$ 
		and the models  of 
	SM$_{{\bf p}}[\fol{\Pi}]$. 
\end{theorem}

This result on correctness of projection is immediate consequences of Lemma on Explicit Definitions presented here.
We  note that the proof  of a more restricted statement by \citeN{har16} is rather complex relying directly on the definition of SM operator. This illustrates the utility of presented theory, e.g., Lemma on Explicit Definitions, as it equips ASP practitioners with a formal result that  eases a construction of proofs of correctness of their rewritings.

\citeN{hip18} considered  rewritings that they call $\alpha$ and $\beta$-projections. They 
also implement these rewritings in 
system~{\sc projector}. Both $\alpha$ and $\beta$-projections are instances of the projection defined here. 
As a result, Theorem~\ref{thm:projection} provides a proof of correctness for the $\alpha$ and $\beta$-projections.

Here we reproduce the definition of $\alpha$-projection by~\citeN[Section 2]{hip18} for the case of {\em positive} rules of the form 
$$
a_0\ar a_1,\dots, a_m,
$$
where $a_0$ is an atom or $\bot$ and $a_1,\dots, a_m$ are atoms (we use the notation of this paper to reproduce the definition).  Expression~\eqref{eq:s1def} exemplifies  a positive rule.
  For a positive rule $\rho$ and a set ${\bf x}$ of variables, by
 $\mathit{alpha}(\rho,{\bf x})$ we denote the  set of all atoms in the body of $\rho$ such that they contain {\em some} variable in ${\bf x}$.
 For instance, let $\rho_1$ be  rule~~\eqref{eq:s1def}.
 Then,
 $$
 \ba{l}
 \mathit{alpha}(\rho_1,\{Y\})=\{q(X,Y),r(X,Y)\}\\
 \mathit{alpha}(\rho_1,\{X,Y\})=\{q(X,Y),r(X,Y),t(X)\}
 \ea
 $$
 
 
 For a set ${\bf x}$ of  variables and a positive rule $\rho$ of the form $H\ar \body$, where 
 no variable in~${\bf x}$  occurs in $H$, the process of {\em $\alpha$-projecting} ${\bf x}$  out of this rule will result in replacing it by two rules:
 \begin{enumerate}[topsep=0pt,noitemsep]
 	\item a rule $$u({\bf y})\ar \mathit{alpha}(\rho,{\bf x}).$$ 
 	so that
 	\begin{itemize}[topsep=0pt,noitemsep]
 		\item $u$ is a fresh predicate symbol with respect to original program, and 
 		\item ${\bf y}$ is composed of the variables that occur in  $\mathit{alpha}(\rho,{\bf x})$, but not in ${\bf x}$; 
 	\end{itemize} 
 	\item a rule $$ H\ar (\body\setminus \mathit{alpha}(\rho,{\bf x}))\cup\{u({\bf y})\}.$$
 \end{enumerate}
 For instance, replacing rule~\eqref{eq:s1def} with rules~\eqref{eq:al1}  and~\eqref{eq:al2} exemplifies $\alpha$-projection of $Y$. 
It is  easy to see that $\alpha$-projection on positive programs  is an instance of projection rewriting studied here. 
The definitions of $\alpha$ and $\beta$-projections for general programs require  substantially more notation. Thus, we refer an interested reader to the paper by~\citeN[Section 2]{hip18} for the details. Yet, it is still easy to see that these rewritings are instances of projection as defined here.
For example,  replacing rule~\eqref{eq:s1def} with rules~\eqref{eq:al3}
 and~\eqref{eq:sprime2def} exemplifies $\beta$-projection.

\vspace{-1.2em}
\section{Conclusions}

We illustrated how the concepts of strong equivalence and conservative extensions can be used jointly to argue the correctness of a newly designed program or correctness of program rewritings. This work outlines a methodology for such arguments.  
Also, this paper lifts several important theoretical results for propositional programs to the case of first-order logic programs. These new formal findings allow us to argue a number of first-order program rewritings to be safe. We illustrate the usefulness of these findings by utilizing them in constructing an argument which shows that the sample programs \alchoice and \dalchoice  are essentially the same. We believe that these results provide a strong building block for a portfolio of safe rewritings that can be used in creating an automatic tool for carrying these rewritings during program performance optimization phase. 
For example, system {\sc projector} discussed in the last section implements projection rewritings for the sake of performance. In this work we utilized the presented formal results to argue the correctness of this system. 


\vspace{-.5em}
\paragraph*{Acknowledgements}
We are grateful to Pedro Cabalar, Jorge Fandinno, Nicholas Hippen, Vladimir Lifschitz,  Miroslaw Truszczynski for valuable discussions on the subject of this paper. We are also thankful to the anonymous reviewers for their help to improve the presentation of the material in the paper.  Yuliya Lierler was partially supported by the NSF 1707371 grant.

\appendix
\section{"Weak" Natural Deduction System}
Recall that in addition to introducing the strong equivalence, \citeN{lif01} also illustrated that traditional programs can be associated with the propositional formulas and a question whether the programs are strongly equivalent can  be turned into a question whether the respective propositional formulas are equivalent in the HT logic.
The authors also  state that
every formula provable in
the natural deduction system, where 
the axiom of the law of the excluded middle~($F\vee\neg F$) is replaced by the weak law of the excluded middle ($\neg F\vee\neg \neg F$), is a theorem of logic HT. We call this system {\em weak natural deduction system}.
Since we use this observation in providing formal arguments stated in Section~\ref{sec:prop}, we review the weak natural deduction system here. We denote this system by {\bf N}.
Its review follows the lines of~\citep{cs388l} to a large extent. For another reference to natural deductions system we refer the reader to~\cite{lif08b}. We note that \citeN{min10} introduced an alternative sequent calculus for logic of HT that was further generalized to first-order case.

A {\sl sequent} is an expression of the form
\beq
\G\seq F
\eeq{seq1}
(``$F$ under assumptions~$\G$''), where $F$ is a propositional formula that allows connectives $$\bot, \top, \neg,\wedge,\vee,\rar$$ and  $\G$ is a finite set of
formulas.  If~$\G$ is written as $\{G_1,\dots,G_n\}$, we  drop
the braces and write~(\ref{seq1}) as
$G_1,\dots,G_n\seq F.$
Intuitively, this sequent  is understood as the formula
$(G_1\wedge\cdots\wedge G_n) \rar F$
if $n>0$, and as $F$ if $n=0$.

The axioms of {\bf N} are sequents of the forms
$$F \seq F,~~~~ \seq \top,
\hbox{ ~~~~and~~~~ }
\seq \neg F\vee \neg\neg F.$$

In the list of inference rules presented in Figure~\ref{fig:infrules}, $\G$, $\D$,~$\Sigma$
are finite sets of
formulas, and~$F,G,H$ are formulas. 
The inference
rules of {\bf N} except for the two rules at the last row are classified into
introduction rules $({\mathbf \cdot} I)$ and {\sl elimination rules}~$({\mathbf \cdot} E)$; the exceptions are the contradiction rule $(C)$ and the
weakening rule $(W)$.

\begin{figure*}[th]
	\begin{minipage}{.96\textwidth}
		
		$$\begin{array}{ll}
		\!\!\!\!\!\!\!\!(\wedge I)\;\r	{\G\seq F \quad \D\seq G}
		{\G,\D\seq F\wedge G}
		&\qquad		
		(\wedge E)\;\r	{\G\seq F \wedge G}
		{\G\seq F}
		\quad
		\r  {\G\seq F \wedge G}
		{\G\seq G}
		\\ \\
		\!\!\!\!\!\!\!\!(\vee I)\;\r	{\G\seq F}
		{\G\seq F \vee G}\quad
		\r	{\G\seq G}
		{\G\seq F \vee G}
		&\qquad				
		(\vee E)\;\r{\G\seq F \vee G \quad \D,F \seq H \quad \Sigma,G\seq H}
		{\G,\D,\Sigma\seq H}
		\\ \\
		\!\!\!\!\!\!\!\!(\rar\!\! I)\;\r	{\G,F\seq G}
		{\G\seq F\rar G}
		&\qquad				
		(\rar\!\! E)\;\r	{\G\seq F\quad \D\seq F \rar G}
		{\G,\D\seq G}
		\\ \\
		\!\!\!\!\!\!\!\!(\neg I)\;\r	{\G,F\seq\bot}
		{\G\seq\neg F}
		&\qquad				
		(\neg E)\;\r	{\G\seq F \quad \D\seq \neg F}
		{\G,\D\seq\bot}
		\\ \\ 
		\!\!\!\!\!\!\!\! (C)\;\r {\G\seq\bot}{\G\seq F}
		&\qquad (W)\;\r {\G\seq H}{\G,\D\seq H}\\ \\
		\end{array}$$
		\vskip -2mm
		\normalsize
		\caption{Inference rules of system {\bf N}. \label{fig:infrules}}
	\end{minipage}
\end{figure*}
A {\em proof/derivation} is a list of sequents
$S_1,\dots,S_n$ such that each $S_i$ is either an axiom or can be derived from some of the sequents in  $S_1,\dots,S_{i-1}$ by one of the inference rules.
To prove a sequent $S$ means to find a proof with the last sequent~$S$.
To prove a formula~$F$ means to prove the sequent $\seq F$.

The De Morgan's law
$$\neg (F\vee G) \lrar \neg F\wedge \neg G$$
is provable intuitionistically  (where we understand formula $H\lrar H'$ as an abbreviation for $(H\rar H')\wedge (H'\rar H)$). Thus, formulas  $\neg (F\vee G)$ and $\neg F\wedge \neg G$ are intuitionistically equivalent.
The other De Morgan's law
$$\neg (F\wedge G) \lrar \neg F\vee \neg G$$
is such that its one half is provable intuitionistically, while the other one is provable in HT (thus, formulas  $\neg (F\wedge G)$ and $\neg F\vee \neg G$ are  equivalent in HT-logic). 
We illustrate the latter fact in Figure~\ref{fig:prove1}  using system {\bf N}.
In other words,  we prove sequent $\seq \neg (F\wedge G) \rar \neg F\vee \neg G$ in {\bf N}. 
It is convenient to introduce abbreviations for the assumptions used in the proofs so that~$A_1$ abbreviates assumption  $\neg(F\wedge G)$ in Figure~\ref{fig:prove1}.
\begin{figure*}[ht]
	\begin{minipage}{.96\textwidth}
		$$
		\begin{array}{rll|rll}
		1.& \seq \neg F\vee \neg\neg F	&\hbox{ axiom} & 8.& \neg \neg F\seq  \neg \neg F		&\hbox{ axiom}\\				
		
		A_1.& \neg(F\wedge G)	&				& 9.& A_1,G, \neg \neg F\seq  \bot		&\hbox{ $(\neg E)$ 7,8}\\		
		
		2.& A_1\seq \neg(F\wedge G)	&\hbox{ axiom} & 10.& A_1, \neg \neg F\seq  \neg G		&\hbox{ $(\neg I)$ 9}\\		
		
		3.& G\seq  G		&\hbox{ axiom}& 11.& A_1, \neg \neg F\seq \neg F\vee \neg G		&\hbox{ $(\vee I)$ 10}\\
		
		4.& F\seq  F		&\hbox{ axiom} & 12.& \neg F\seq  \neg F		&\hbox{ axiom}\\		
		
		5.& F,G\seq  F\wedge G		&\hbox{ $(\wedge I)$ 3,4} & 13.& \neg F\seq  \neg F	\vee\neg G	&\hbox{ $(\vee I)$ 12}\\		
		
		6.& A_1,F,G\seq \bot		&\hbox{ $(\neg E)$ 2,5} & 14.& A_1\seq  \neg F \vee\neg G	&\hbox{ $(\vee E)$ 1, 11, 13}\\
		
		7.& A_1,G\seq \neg F		&\hbox{ $(\neg I)$ 6} & 15.& \seq \neg(F\wedge G) \rar (\neg F \vee\neg G)	&\hbox{ $(\rar I)$ 14}\\ 				
		\end{array}
		$$
		\vskip -2mm
		\normalsize
		\caption{Proof of sequent $\seq \neg (F\wedge G) \rar \neg F\vee \neg G$ in system {\bf N}. \label{fig:prove1}}
	\end{minipage}
\end{figure*}

It is easy to show that the propositional formulas $F\rar \bot$ and $\neg F$ are equivalent using {\bf N}, so that in the sequel we often identify rules of the form
$$	 a_{1} \wedge \dots \wedge a_l \wedge \neg  a_{l+1} \wedge \dots 
\wedge \neg a_m \wedge \neg \neg a_{m+1} \wedge \dots \wedge \neg\neg  a_n \rar \bot  
$$
with the propositional formula
$$
\neg(  a_{1} \wedge \dots \wedge a_l \wedge  a_{l+1} \wedge \dots 
\wedge \neg a_m \wedge \neg \neg a_{m+1} \wedge \dots \wedge \neg \neg a_n).
$$

\section{Proofs}
\subsection{Proofs for Section~\ref{sec:prop}}
To prove Proposition~\ref{prop:ce} several earlier results from the literature are of use.

\begin{prop}[Replacement Theorem I in \citep{min00}, Section 2.8]\label{prop:repl} If $F$ is a formula  containing a subformula $G$ and $F'$ is the result of replacing that subformula by~$G'$ then $G\lrar G'$ intuitionistically implies $F\lrar F'$. 
\end{prop}

To rely on formal results stated earlier in the literature, we now consider the case of programs that are more general than traditional logic programs. We call such programs {\em definitional}. 
In other words, traditional programs are their special case. 
A {\em definitional program}  consists of rules of the form~\eqref{eq:formulatraditional} (recall that we identify rule~\eqref{eq:ruletraditional} with the propositional formula~\eqref{eq:formulatraditional})
and rules of the form $a\rar F$, where $a$ is an atom and $F$ is a basic conjunction. If a program contains two rules $F\rar a$ and $a\rar F$ we abbreviate that by a single expression $F\lrar a$. A definitional program is a special case of propositional theories presented in~\citep{fer05}.   We understand answer sets for definitional programs as presented there. 
 \citeN[Section 2]{fer05} illustrates that in application to any traditional program the definition from~\cite{lif99d}, presented here, and their definition are equivalent.

We now restate the results that immediately follow from Lemma on Explicit Definitions and Completion Lemma presented in~\citep{fer05} for the case of definitional programs. 

\begin{prop}[Proposition 4~\citep{fer05}]\label{prop:expdeffer} 
	Let $\Pi$ be a definitional program and $Q$ be a set of atoms that do not occur in $\Pi$. For each $q\in Q$, let $Def(q)$ be a basic conjunction that does not contain any atom in $Q$. Then, $X\mapsto X\setminus Q$ is a 1-1 correspondence between the answer sets if $\Pi \cup \{ Def(q)\rar q\mid q\in Q\}$ and the answer sets of $\Pi$.
	\end{prop}
\begin{prop}[Proposition 5~\citep{fer05}]\label{prop:complemmafer} 
	Let $\Pi$ be a definitional program and $Q$ be a set of atoms that do not occur in $\Pi$. For each $q\in Q$, let $Def(q)$ be a basic conjunction that does not contain any atom in $Q$. Then, $\Pi \cup \{Def(q)\rar q\mid q\in Q\}$ and $\Pi \cup \{ Def(q)\lrar q\mid q\in Q\}$ have the same answer sets.	
\end{prop}

\begin{proof}[Proof of Proposition \ref{prop:ce}] 
	By $\Pi'$ we denote a program constructed from $\Pi$ by  adding a  rule 
	$def(q)\rar q$ for every atom \hbox{$q\in Q$}. By Proposition~\ref{prop:expdeffer}, $\Pi'$ is a conservative extension of $\Pi$. 
	By Proposition~\ref{prop:complemmafer}, traditional program $\Pi'$ has the same answer sets as the definitional program $\Pi''$ constructed from $\Pi'$ by replacing a rule 	$def(q)\rar q$ with a rule $def(q)\lrar q$ . 
	Similarly,  traditional program $\Pi[Q,def(Q)]$ has the same answer sets as the definitional program $\Pi[Q,def(Q)]'$ constructed from it by replacing a rule 	$def(q)\rar q$ with a rule $def(q)\lrar q$.  
	By Replacement Theorem I, $\Pi''$ and   $\Pi[Q,def(Q)]'$ are strongly equivalent.
\end{proof}

We now state auxiliary lemmas that are useful in the argument of Proposition~\ref{prop:main}.
It is constructed by uncovering the formal link between logic programs  $simp_T(D)$ and $lp_T(D)$, where $lp_T(D)$ serves the role of a gold standard by the virtue of Proposition~\ref{prop:lif}. 

\begin{lemma}\label{lem:aux1}
	If $F$ is a formula  containing a subformula $G$ and~$F'$ is the result of replacing that subformula by~$G'$ then
	the sequent
	$$
	\Gamma\seq (G\lrar G')\rar (F\lrar F')
	$$
	is provable in {\bf N}, where $\Gamma$ is arbitrary set of assumptions.
\end{lemma}
\begin{proof} Trivially follows from the Replacement Theorem I stated as Proposition~\ref{prop:repl} here.
\end{proof}

\begin{lemma}\label{lem:aux} The sequent
	$$
	\neg(F\wedge G), \neg(\neg F\wedge \neg G) \seq \neg F\lrar \neg\neg G \wedge \neg G\lrar \neg\neg F
	$$
	is provable in {\bf N}.
\end{lemma}
\begin{proof}
	We illustrate the proof in {\bf N} for the sequent
	$$\neg(F\wedge G)\seq\neg \neg F\rar\neg G.$$
	We allow ourselves a freedom to use De Morgan's Laws as if they were given to us as additional inference rules in {\bf N}. 
	$$
	\begin{array}{rll}
	A_1.&\neg (F\wedge G)&\\
	1.& A_1\seq \neg (F\wedge G)	&\hbox{ axiom} \\
	2.& A_1\seq \neg F\vee \neg G	&\hbox{De Morgan's Law 1} \\
	3.& \neg F\seq \neg F	&\hbox{axiom} \\
	4.& \neg G\seq \neg G	&\hbox{axiom} \\
	5.& \neg\neg F\seq \neg \neg F	&\hbox{axiom} \\
	6.& \neg\neg F, \neg F\seq \bot	& \hbox{($\neg$ E) 3,5} \\
	7.& \neg\neg F, \neg F\seq \neg G	& \hbox{(C) 6} \\
	8.& A_1,\neg\neg F \seq \neg G	& \hbox{($\vee$ E) 2,4,7} \\
	9.& A_1\seq \neg\neg F \rar\neg G	& \hbox{($\rar$ I) 8} 
	\end{array}
	$$
	\normalsize
	Similar proofs in structure are available for  the sequents
	$$
	\ba{l}
	\neg(F\wedge G)\seq \neg\neg G\rar\neg F,\\
	\neg(\neg F\wedge \neg G)\seq \neg F\rar\neg \neg G, \hbox{ and}\\
	\neg(\neg F\wedge \neg G)\seq \neg G\rar\neg \neg F.
	\ea
	$$
	Several applications of $(\wedge I)$ will allow us  to conclude the proof in {\bf N} for the sequent in the statement of this lemma.
\end{proof}

\begin{proof}[Proof of Proposition \ref{prop:main}] 
	It is easy to see that the signatures of $simp_T(D)$ and $lp_T(D)$ differ by complement action atoms present in $lp_T(D)$.
	What we show next is the fact that $lp_T(D)$ is a conservative extension of $simp_T(D)$. Then the claim of this proposition follows from Proposition~\ref{prop:lif}. 
	
	{\em Claim 1}: 
	The set of rules from groups 1 and 3 of $lp_T(D)$ are strongly equivalent to the set of  rules from group 1 of $lp_T(D)$ and 
	the rules 
	\beq
	\ba{l}
	\h{l_0}(t+1) \ar not\ not\ \h{l_1}(t+1),\dots,
	not\ not\ \h{l_m}(t+1),\\
	~~~~~~~~~~~~~~~~~~~~~~~~~~~\h{l_{m+1}}(t),\dots,\ \h{l_n}(t), 
	\ea
	\eeq{eq:dynamicintermid}
	for all $t=0,\dots,T-1$,
	for every dynamic law~\eqref{eq:dlaw} in $D$.    
	
	\medskip  
	
	It is easy to see that these sets of rules only differ in structure of rules~\eqref{eq:dynamicorig} and~\eqref{eq:dynamicintermid} so that the atoms of the form $\o{l_i}(t+1)$ $(1\leq i\leq m)$ in~\eqref{eq:dynamicorig} are replaced by the expressions of the form $not\ \h{l_i}(t+1)$ in~\eqref{eq:dynamicintermid}.

	\begin{figure*}[t!]\
		\begin{minipage}{.96\textwidth}
			\begin{center}
				$$
				\ba{lll}
				\hbox{Sequent 1}&
				\Gamma\seq &
				(\neg \o{l_1}(t+1)\wedge\cdots\wedge\neg\o{l_m}(t+1)\wedge 
				\h{l_{m+1}}(t)\wedge\cdots\wedge \h{l_n}(t)\rar \h{l_0}(t+1))
				\lrar\\
				&  &  	(\neg\neg \h{l_1}(t+1)\wedge\cdots\wedge\neg\neg\h{l_m}(t+1)\wedge 
				\h{l_{m+1}}(t)\wedge\cdots\wedge \h{l_n}(t)\rar \h{l_0}(t+1))\\
				\hbox{Sequent 2}&
				\seq& 
				\Gamma\wedge (\neg \o{l_1}(t+1)\wedge\cdots\wedge\neg\o{l_m}(t+1)\wedge 
				\h{l_{m+1}}(t)\wedge\cdots\wedge \h{l_n}(t)\rar \h{l_0}(t+1))
				\lrar\\
				& & \Gamma\wedge (\neg\neg \h{l_1}(t+1)\wedge\cdots\wedge\neg\neg\h{l_m}(t+1)\wedge 
				\h{l_{m+1}}(t)\wedge\cdots\wedge \h{l_n}(t)\rar \h{l_0}(t+1)).
				
				\ea
				$$
			\end{center}
			
		\end{minipage}
		\caption{Sequents in the proof of Proposition~\ref{prop:main}\label{fig:sequents}}
	\end{figure*}
	
	Let $\Gamma$ denote the set of rules from group 1 of $lp_T(D)$. Using Lemmas~\ref{lem:aux1} and~\ref{lem:aux} it is easy to see that the  sequent 1 presented in Figure~\ref{fig:sequents}
	is provable in {\bf N}. It is easy to construct a proof in {\bf N} from this sequent 1 to the sequent 2 in the same figure.
	This immediately concludes the proof of Claim 1.
	
	\medskip  
	
	{\em Claim 2}: 
	The set of rules from groups 1 and 2 of $lp_T(D)$ are strongly equivalent to the set of  rules from group 1 and~2  of $simp_T(D)$. 
	
	\medskip  
	The proof for this claim follows the lines of a proof for Claim 1.
	
	\medskip  
	{\em Claim 3}: 
	The set of rules from groups 1 and 4 of $lp_T(D)$ are strongly equivalent to the set of  rules from group 1 and~4  of $simp_T(D)$. 
	
	\medskip  
	The proof for this claim is similar to that of a proof of Claim 1.

	\medskip
	Due to Claims 1, 2, and 3, it follows that $lp_T(D)$ has the same answer sets as the program $lp'_T(D)$ constructed from $lp_T(D)$ by replacing (i) the rules  from group  3 with rules~\eqref{eq:dynamicintermid}
	for all $t=0,\dots,T-1$,
	for every dynamic law~\eqref{eq:dlaw} in $D$, and
	(ii) the rules  from groups  2 and 4 in $lp_T(D)$ by the rules from groups 2 and 4 in $simp_T(D)$.    
	
Let $def(\t{a})$ denote formula $\neg a$ for every elementary action $a$ in $\sigma^{act}_T$.	
It is easy to see that $lp'_T(D)$
	coincides with the program 
	$$
	\ba{l}
	simp_T(D)[~\{\t{a}\mid a\in \sigma^{act}_T\},~
	\{def(\t{a}) \mid a\in \sigma^{act}_T  \}~].
	\ea$$
	By Proposition~\ref{prop:ce}, program $lp'_T(D)$  is a conservative extension of $simp_T(D)$. Consequently, $lp_T(D)$ is a conservative extension of $simp_T(D)$.
\end{proof}

\subsection{Proofs for Section~\ref{sec:foprogram}}
\begin{proof}[Proof of Theorem~\ref{thm:Aggregates}]
	Consider the case when $H$ is a disjunction of atoms then FOL representation of rule~\eqref{eq:ruleag} is the universal closure of formula
\beq
	\Big(\exists \vec{x^1}\cdots\vec{x^b} \big( \bigwedge_{1\leq i\leq b} F(\vec{x^i}) 
	\wedge\bigwedge_{1\leq i<j\leq b} \neg(x^i=x^j)\big) \wedge G
	\Big) \rar H.
\eeq{eq:formulainproof}	
	The FOL representation of rule~\eqref{eq:rulenoag} is the universal closure of formula
\[
	\big( \bigwedge_{1\leq i\leq b} F(\vec{x^i}) 
	\wedge\bigwedge_{1\leq i<j\leq b} \neg(x^i=x^j)
		\wedge  G \big) \rar H. 
\]
	Given that  formula $\forall \vec z (H \rar H')$ where $H'$ has no free occurrences of variables in $\vec{z}$  
	is intuitionistically equivalent to $\exists \vec z (H) \rar H'$,
	the FOL representation of rule~\eqref{eq:rulenoag} can be written as the universal closure of formula
	\[
	\Big(\exists \vec{x^1}\cdots\vec{x^b}  \big(\bigwedge_{1\leq i\leq b} F(\vec{x^i}) 
	\wedge\bigwedge_{1\leq i<j\leq b} \neg(x^i=x^j)
	\wedge  G \big)  \Big)\rar H. 
	\]
	
It is easy to see that the left hand sides of the implications in this formula and formula~\eqref{eq:formulainproof} are classically equivalent. 
And thus by Replacement Theorem II these formulas are intuitionistically equivalent.
Similarly we can argue for the case when $H$ is of the form $\{a\}$.
 \end{proof}
 
\begin{proof}[Outline of the Proof of Theorem~\ref{thm:choice}]
  We can 
  derive 
   the former program given in the theorem statement (its FOL representation) from the latter intuitionistically; 
    and   	 we can 
    	 derive 
    the later  from the former in logic \sqht. For the second direction,
    De Morgan's law $\neg (F\wedge G)\rar \neg F\vee \neg G$ (provable in logic \sqht, but not valid intuitionistically) is essential. 
 \end{proof}
 
     To prove Theorem~\ref{thm:shift2} we recall the Splitting Lemma from~\citep{fer09a} (this Splitting Lemma is the generalization of the Splitting Set Theorem~\citep{lif94e}).
     
     \smallskip
     \noindent
     {\bf Splitting Lemma.}~~
     {\em 	Let $F$ be a first-order sentence.
     	
     	{\em Version 1}:	
     	Let~{\bf p, q} be disjoint tuples of distinct predicate constants. If each strongly connected component of  $DG_{\bf pq}[F]$ is a subset of {\bf p} or a subset of~{\bf q}, then 
     	$\SM_{\bf pq}[F]$ is equivalent to \hbox{$\SM_{\bf p}[F]\wedge\SM_{\bf q}[F]$}. 
     	
     	{\em Version 2}:
     	Let~{\bf p} be a tuple of distinct predicate constants. If ${\bf c^1,\dots,c^n}$ are all the strongly connected components of 
     	$DG_{\bf p}[F]$, then 
     	$\SM_{\bf p}[F]$ is equivalent to \hbox{$\SM_{\bf c^1}[F]\wedge\cdots\wedge\SM_{\bf c^n}[F]$}. 
     }
     \smallskip
     
     \noindent
     \begin{proof}[Proof of Theorem~\ref{thm:shift2}]
     	We start by partitioning~$C$ into two sets   $Q$ and~${\bf r}$ so that 
     	\begin{itemize}
     		\item any element in $Q$ is  such that at least one of its predicate symbols occurs in a head of some rule in  $\bf R$,		
     		\item any element in ${\bf r}$ is  such that none of its predicate symbols occurs in a head of some rule in  $\bf R$.
     	\end{itemize}
     	We identify set ${\bf r}$ with a tuple (order is immaterial) composed of the predicate symbols occurring in its elements. For set $Q$ we identify every strongly connected component  ${\bf q}\in Q$, with a tuple of predicate symbols in this component.
     	
     	By $\Pi^{sh}$ we denote	a program constructed from $\Pi$ by replacing each rule $R\in \bf R$ with  $\shift_{\mathcal{P}^R}(R)$.

     	By definition, an answer set of $\Pi$ is an Herbrand model of formula~\eqref{eq:smpi}.
     	Similarly, an answer set of $\Pi^{sh}$   is an Herbrand  model of \beq
     	\SM_{\bf \pred(\fol{\Pi^{sh}})}[\fol{\Pi^{sh}}].
     	\eeq{eq:thmprove2} We now show that formulas~\eqref{eq:smpi} and
     	~\eqref{eq:thmprove2} are equivalent.

     	By the Splitting Lemma, formula~\eqref{eq:smpi}  
     	is  equivalent to 
     	\beq\SM_{\bf {r}}[\fol{\Pi}] \wedge \bigwedge_{{\bf q}\in Q}\SM_{{\bf q{}}}[\fol{\Pi}]\eeq{eq:bigshift}
     	Theorem~5 from \citep{fer09} shows that
     	given formulas $\SM_{{\bf p}}[F]$ and $\SM_{{\bf p}}[G]$ so that $\pred(F)=\pred(G)$
     	if the equivalence between $F$ and~$G$  can 
     	be derived intuitionistically from the law of excluded middle formulas for all  predicates in $\pred(F)\setminus{\bf p}$, then they have the same stable  models. Following claims are  the consequences of that theorem 
     	\begin{itemize}
     		\itemsep0em 
     		\item $\SM_{{\bf r}}[\fol{\Pi}]$ is equivalent to  $\SM_{{\bf r}}[\fol{\Pi^{sh}}]$,
     		\item for every ${\bf q}$ in $Q$,  $\SM_{{\bf q}}[\fol{\Pi}]$ is equivalent to  $\SM_{{\bf q}}[\fol{\Pi^{sh}}]$.  	 	
     	\end{itemize}
     	Consequently, formula~\eqref{eq:bigshift} is 
     	equivalent to formula
     	\beq\SM_{{\bf r}}[\fol{\Pi^{sh}}] \wedge \bigwedge_{{\bf q}\in Q}\SM_{{\bf q}}[\fol{\Pi^{sh}}]\eeq{eq:bigshift2}
     	It is easy to see that  $\pred(\fol{\Pi})=\pred(\fol{\Pi^{sh}})$. 
     	By the Splitting Lemma,  formula~\eqref{eq:bigshift2} is  equivalent to~\eqref{eq:thmprove2}. 
     \end{proof}
     
 In order to state a proof for Completion Lemma, we recall several important theorems from~\cite{fer09a,fer09}. 
 \begin{theorem}[Splitting Theorem~\citep{fer09a}]\label{thm:splitfer}
 	Let $F$ and $D$ be first-order sentences, and let ${\bf p,q}$ be disjoint tuples of distinct predicate constants. If
 	\begin{itemize}
 		\item each strongly connected component of $DG_{\bf p,q}[F\wedge D]$ is a subset of {\bf p} or {\bf q},
 		\item members of {\bf p} have no strictly positive occurrences in $D$, and 
 		\item members of {\bf q} have no strictly positive occurrences in $F$
 	\end{itemize}
 	then 
 	$$
 	\SM_{\bf pq}[F\wedge D]\hbox{ is equivalent to } \SM_{\bf p}[F]\wedge \SM_{\bf q}[D].
 	$$
 \end{theorem}
 
 \begin{theorem}[Theorem~2~\citep{fer09}]\label{thm:two}
 	Let $F$ be first-order sentences, and let ${\bf p,q}$ be disjoint tuples of distinct predicate constants. 
 	Then 
 	$$
 	\SM_{\bf pq}[F\wedge \bigwedge_{q\in\{\bf q\}} {\forall \vec{x} 
 		\big( \neg \neg q(\vec x)\rar 
 		q(\vec{x})\big)}]\hbox{ is equivalent to } \SM_{\bf p}[F].
 	$$
 \end{theorem}

 We say that formula $\SM_{\bf p}[F]$ is {\em tight} if the graph   $DG_{\bf p}[F]$ is acyclic.
 \begin{theorem}[Theorem 11~\citep{fer09a}]\label{thm:completion}
 	For any tight formula $\SM_{\bf p}[F]$ where $F$ is in Clark normal form relative to {\bf p}, $\SM_{\bf p}[F]$ is equivalent to the completion of $F$.
 \end{theorem}
 
 \begin{proof}[Proof of Theorem~\ref{thm:complemma}]
 Recall that $D$ and $Comp[D]$ denote $$
 \ba{ll}
 \bigwedge_{q \hbox{ in \bf q}}\forall\bf x \big(G\rar q(\vec{x})\big)\hbox{ ~ and } & \bigwedge_{q \hbox{ in \bf q}}\forall\bf x \big(G\lrar q(\vec{x})\big),
 \ea$$
 respectively.
 From the conditions posed on the occurrence of elements in {\bf q} in $F$ and $D$ it is easy to see that every element in {\bf q} forms a singleton strongly connected component in $DG_{\bf p,q}[F\wedge D]$ and in $DG_{\bf q}[D]$.
 Consequently, each strongly connected component of $DG_{\bf p,q}[F\wedge D]$ is a subset of {\bf p} or {\bf q} and $D$ is tight.  Furthermore, by the assumptions of the theorem,
  members of {\bf p} have no strictly positive occurrences in $D$, and 
  members of {\bf q} have no strictly positive occurrences in~$F$.
 By Theorem~\ref{thm:splitfer}, formula~\eqref{eq:goal}
 is equivalent to 
 $\SM_{\bf p}[F]\wedge\SM_{\bf q}[D]$.
 By Theorem~\ref{thm:completion}, $\SM_{\bf q}[D]$ is equivalent to $Comp[D]$. Consequently, formula~\eqref{eq:goal} is equivalent to~\eqref{goal1} and~\eqref{goal2}. By Theorem~\ref{thm:two},
 formula~\eqref{goal2} is equivalent to formula~\eqref{goal3}.
 \end{proof}
 
 \begin{proof}[Proof of Theorem~\ref{thm:complemma2}]
 	By Theorem~\ref{thm:complemma2},
 	$\SM_{\bf pq}[F\wedge D]$ is equivalent to 
 		$\SM_{\bf p}[F]\wedge Comp[D]$. Formula $Comp[D]$
  corresponds to so called explicit definitions of predicates in {\bf q}.
  There is an obvious 1-1
  correspondence between the models of $\SM_{\bf p}[F]$ and the models of the
  same formula extended with explicit definitions (for predicates that do not occur in $F$). In particular, if $M$ is a model of $\SM_{\bf p}[F]\wedge Comp[D]$ then $M_{|\sigma(F)}$ is a model of $\SM_{\bf p}[F]$.
  This concludes the proof of statement (i).
  By the Replacement Theorem for intuitionistic logic, we conclude that 
  	$\SM_{\bf p}[F]\wedge Comp[D]$
  	is equivalent to $\SM_{\bf p}[F^{\bf q}]\wedge Comp[D]$
  and hence to 
 	$\SM_{\bf pq}[F^{\bf q}\wedge D]$ 	by Theorem~\ref{thm:complemma2}.
 	This concludes the proof of statement (ii). Statement (iii) is proved similarly.

 \end{proof}

\bibliographystyle{acmtrans}
\bibliography{shared}

\end{document}

\citeN{hip18} considered a rewriting technique in spirit of projection defined here. They describe two rewritings $\alpha$ and $\beta$-projections. 
For instance, replacing rule~\eqref{eq:s1def} with rules~\eqref{eq:al1}  and~\eqref{eq:al2} exemplifies $\alpha$-projection. Replacing rule~\eqref{eq:s1def} with rules~\eqref{eq:al3}
and~\eqref{eq:sprime2def} exemplifies $\beta$-projection.
Theorem~\ref{thm:projection} provides grounds for a proof of correctness for these projections. System~{\sc projector} described in~\cite{hip18} implements these rewritings.

We now generalize the notion of a result of projecting {\bf x} out of RASPL-1 rule $R$ also when some variables in {\bf x} occur in   aggregate expressions.


Let~$R$ be a RASPL-1 rule of the form
\beq H\ar b\leq \#count\{\vec{y}:F\}, \body
\eeq{eq:rasplprojrule}
Rule $R$ occurs in a program~$\Pi$. Let 
{\bf x} be a non-empty tuple of variables occurring
only in 
body of~$R$ such that 
no member of {\bf x} is an aggregate variable of some aggregate expression.
By $\beta({\bf x})$ we denote some set  of literals in the body of~$R$ outside of any aggregate expressions so that it includes all literals that have occurrences of variables in~{\bf x}.
By $\gamma({\bf x})$ we denote some set of  literals in $b\leq \#count\{\vec{y}:F\}$  so that it includes all literals that have occurrences of variables in~{\bf x}.
Let $u$ be a 
predicate symbol that does not occur in $\Pi$. Then a
\emph{result of projecting variables~${\bf x}$ out of~$R$ using predicate symbol $u$} consists  of the 
following two rules  
\begin{flalign}
&H\ar b\leq \#count\{\vec{y}:u({\bf z}),F\setminus\gamma({\bf x})\},\body.\label{rule:proj1}
\\
&u({\bf z})\ar\beta(\bf x)\cup\gamma(\bf x).\label{rule:proj2}
\end{flalign}
where $\bf z$ denotes the variables that occur in $\beta(\bf x)\cup\gamma(\bf x)$, but do not occur in {\bf x}.

Consider RASPL-1 rule 
$$
\begin{array}{ll}
good(X)\ar& vtx(X),2\leq\#count\{y:edge(X,y),edge(y,Z),red(Z)\},good(Z),\\
&2\leq\#count\{Y:edge(X,Y),edge(Y,Z),yellow(Z)\}.
\end{array}
$$
A
result of projecting variable~$Z$ out of~it using predicate symbol $u$ consists  of the 
following two rules:
$$
\begin{array}{l}
good(X)\ar vtx(X),2\leq\#count\{Y:edge(X,Y),u(Y)\},good(Z),\\
~~~~~~~~~~~~~~~~~~~~2\leq\#count\{Y:edge(X,Y),edge(Y,Z),yellow(Z)\}
\\
u(Y) \ar e(Y,Z),red(Z),good(Z).
\end{array}
$$


By $\beta'$ we denote all literals in $\beta(\vec{x})$ that contain no other variables, but these occurring in~$\vec{x}$.

If no variable in {\bf x}  occurs in any aggregate expression of $\body$, then we can simplify rule~\eqref{rule:proj1} as follows
$$
H\ar b\leq \#count\{\vec{y}:u({\bf z}), F\setminus\gamma({\bf x})\},
\body\setminus\beta'.
$$
For instance, recall rule~\eqref{eq:localex.2}.
$$
\begin{array}{ll}
good(X)\ar& vtx(X),2\leq\#count\{Y:e(X,Y),e(Y,Z),red(Z)\},good(Z).
\end{array}
$$
A
result of projecting variable~$Z$ out of~it using predicate symbol $u$ consists  of the 
following two rules 
$$
\begin{array}{l}
good(X)\ar vtx(X),2\leq\#count\{Y:e(X,Y),u(Y)\}\\
u(Y)\ar  e(Y,Z),red(Z),good(Z).
\end{array}
$$

Consider another example 
$$
\begin{array}{l}
\ar\#count\{X,Y:p(Z),q(Z,Y),q(U,T),q(T,S),f(V+1,W),X=2*W\},t(Y),t(Z),t(V),t(S).
\end{array}
$$
One
result of projecting variables~$\{U,T\}$ out of this rule using predicate symbol $u$ consists  of the 
following two rules 
$$
\begin{array}{l}
\ar\#count\{X,Y:u(Z,S),f(V+1,W),X=2*W\},t(Y),t(Z),t(V),t(S)\\
u(Z,S)\ar p(Z),q(Z,Y),q(U,T),q(T,S).
\end{array}
$$
Another
result of projecting variables~$\{U,T\}$ using predicate symbol $u$ consists  of the 
following two rules 
$$
\begin{array}{l}
\ar\#count\{X,Y:u(Z,S),f(V+1,W),X=2*W\},t(Y),t(Z),t(V),t(S)\\
u(Z,S)\ar p(Z),q(Z,Y),q(U,T),q(T,S),t(Z),t(S).
\end{array}
$$

We are now ready to state a formal result stating that applying  described projection technique  results in a program that is essentially equivalent to an original one. We note that already mentioned system {\sc projector} implements rewritings on rules with aggregates as exemplified here. The statement below provides a proof of correctness for this system. 

\begin{theorem}\label{thm:projection2} Let $\Pi$ be a RASPL-1 program, {\bf p} be a set of predicate constants, and $R$ be a RASPL-1 rule of the form
	$$H\ar b\leq \#count\{\vec{y}:F(\vec{y})\}, \body.$$
	Let 
	{\bf x} be a non-empty tuple of variables occurring  only in 
	body of~$R$ such that no member of {\bf x} is an aggregate variable of some aggregate expression. 
	If~$\Pi'$ is constructed from $\Pi$ by replacing~$R$ in~$\Pi$ with a result of 
	projecting variables ${\bf x}$ out of~$R$ using a predicate symbol~$u$ 
	that is not in the signature of $\Pi$,
	then  $M\mapsto M_{|\sigma(\fol{\Pi})}$  is a 1-1 correspondence between the models of 
	SM$_{{\bf p}, u}[\fol{\Pi'}]$ 
	and the models  of 
	SM$_{{\bf p}}[\fol{\Pi}]$. 
\end{theorem}

\begin{proof}[Proof of Theorem~\ref{thm:projection2}]
	Let 
	\begin{itemize}
		\item $\beta({\bf x})$ be some set  of literals in the body of~$R$ outside of any aggregate expressions so that it includes all literals that have occurrences of variables in~{\bf x}.
		\item $\gamma({\bf x})$ be some set of  literals in $b\leq \#count\{\vec{y}:F(\vec{y})\}$  so that it includes all literals that have occurrences of variables in~{\bf x}.
		\item  $\bf z$ be the variables that occur in $\beta({\bf x})\cup \gamma({\bf x})$, but do not occur in~{\bf x}.
	\end{itemize}
	A
	result of projecting variables~${\bf x}$ out of~$R$ using  $u$ consists  of the 
	two rules:  
	\begin{flalign}
	&H\ar b\leq \#count\{\vec{y}:u({\bf z}),F(\vec{y})\setminus\gamma({\bf x})\},\body.\label{l:rp1}\\
	&u({\bf z})\ar\beta({\bf x})\cup \gamma({\bf x})\label{l:rp2}.
	\end{flalign}

	The FOL-representation of $R$ follows:
	\beq
	\univ\Big(\fol{\body}\wedge 
	\exists \vec{y^1}\cdots  \vec{y^b} [\bigwedge_{1\leq i\leq b} F(\vec{y^i}) 
	\wedge\bigwedge_{1\leq i<j\leq b} \neg(y^i=y^j)
	]
	\rar\fol{H} \Big).
	\eeq{eq:folruleproof1}
	
	The LHS of the implication in~\eqref{eq:folruleproof1} is classically equivalent to
	\beq\fol{\beta(\vec{x})}\wedge \fol{\body}\wedge 
	\exists \vec{y^1}\cdots  \vec{y^b} [
	\fol{\gamma(\vec{x})}\wedge
	\bigwedge_{1\leq i\leq b} F(\vec{y^i}) 
	\wedge\bigwedge_{1\leq i<j\leq b} \neg(y^i=y^j)
	].
	\eeq{eq:folruleproof2}
	Since $\vec{y'}$ consists of  local variables, it is obvious that $\beta(\vec{x})$ does not contain any variables in $\vec{y'}$. Also none of the variables $\vec{y^1}\cdots  \vec{y^b}$ are in $\vec{x}$.
	Thus, formula~\eqref{eq:folruleproof2} is classically equivalent to 
	$$
	\fol{\body}\wedge 
	\exists \vec{y^1}\cdots  \vec{y^b} \forall \vec{y'}[\fol{\beta(\vec{x})}\wedge 
	\fol{\gamma(\vec{x})}\wedge
	\bigwedge_{1\leq i\leq b} F(\vec{y^i},\vec{y'}) 
	\wedge\bigwedge_{1\leq i<j\leq b} \neg(y^i=y^j)
	].
	$$
	By  Replacement Theorem II, 
	formula~\eqref{eq:folruleproof1} is intuitionistically equivalent to  formula
	\beq
	\univ\Big(\fol{\body}\wedge 
	\exists \vec{y^1}\cdots  \vec{y^b} \forall \vec{y'}[
	\fol{\beta(\vec{x})}\wedge\fol{\gamma(\vec{x})}\wedge
	\bigwedge_{1\leq i\leq b} F(\vec{y^i},\vec{y'}) 
	\wedge\bigwedge_{1\leq i<j\leq b} \neg(y^i=y^j)
	]
	\rar\fol{H} \Big).
	\eeq{eq:folruleproof3}
	Lemma on Explicit Definitions allows us now to conclude that
	$M\mapsto M_{|\sigma(\fol{\Pi})}$  is a 1-1 correspondence between the models of 
	SM$_{{\bf p}, u}[\fol{\Pi'}]$ 
	and the models  of 
	SM$_{{\bf p}}[\fol{\Pi}]$. 
	This is easy to see by considering the FOL representations of rules~\eqref{l:rp1} and \eqref{l:rp2}, and the relation between the FOL representation of rule~\eqref{l:rp1} and formula~\eqref{eq:folruleproof3}. 
	
	It is easy to see how a similar argument can be constructed for the suggested simplification.
\end{proof}